\documentclass{article}
\usepackage{ijcai23}

\usepackage{times}
\usepackage{soul}
\usepackage{url}
\usepackage{xspace}
\usepackage{xcolor}
\usepackage[utf8]{inputenc}
\usepackage[small]{caption}
\usepackage{graphicx}
\usepackage{amsmath}
 \usepackage{subfigure}
\usepackage{amsthm}
\usepackage{tikz}
\usepackage{booktabs}
\usepackage{multirow}
\usepackage{stackengine}
\usepackage{accents}
\usepackage[ruled,vlined,linesnumbered,norelsize]{algorithm2e}
\usepackage[switch]{lineno}
\usepackage{dsfont}
\usepackage{amssymb}
\usepackage{cleveref}

\usepackage{thmtools} 
\usepackage{thm-restate}
\usepackage{pgfplots}
\usepackage{mathtools}

\usepackage{latexsym}

\newtheorem{remark}{Remark}

\newcommand{\x}[1]{\thetav^{(#1)}}

\newcommand{\FL}{\texttt{FL}\xspace}
\newcommand{\FedAvg}{\texttt{FedAvg}\xspace}
\newcommand{\VRed}{\texttt{VRed}\xspace}
\newcommand{\SemiVRed}{\texttt{Semi-VRed}\xspace}

\newcommand{\PropFair}{\texttt{PropFair}\xspace}
\newcommand{\TERM}{\texttt{TERM}\xspace}
\newcommand{\AFL}{\texttt{AFL}\xspace}
\newcommand{\qFFL}{\texttt{q-FFL}\xspace}
\newcommand{\GiFair}{\texttt{GiFair}\xspace}
\newcommand{\Ditto}{\texttt{Ditto}\xspace}
\newcommand{\MV}{\texttt{MV}\xspace}
\newcommand{\MSV}{\texttt{MSV}\xspace}
\newcommand{\CVaR}{\texttt{CVaR}\xspace}

\newcommand{\Sc}{\mathcal{S}}

\definecolor{mydarkred}{rgb}{0,0,0}

\usepackage{color}
\usepackage{enumerate}
\usepackage{graphicx}
\usepackage{amsfonts, amsmath, bm, amssymb}
\usepackage{dsfont}

\usepackage{xspace}
\makeatletter
\DeclareRobustCommand\onedot{\futurelet\@let@token\@onedot}
\def\@onedot{\ifx\@let@token.\else.\null\fi\xspace}

\makeatother

\newcommand{\thetav     }{\boldsymbol \theta     }

\title{Semi-Variance Reduction for Fair Federated Learning}

\author{
    Saber Malekmohammadi, Yaoliang Yu
    \affiliations
    Cheriton School of Computer Science\\
    University of Waterloo, Canada
    \emails
    saber.malekmohammadi@uwaterloo.ca
}

\begin{document}
\maketitle

\begin{abstract}
    Ensuring fairness in a Federated Learning (\texttt{FL}) system, i.e., a satisfactory performance for all of the participating diverse clients, is an important and challenging problem. There are multiple fair \FL algorithms in the literature, which have been relatively successful in providing fairness. However, these algorithms mostly emphasize on the loss functions of worst-off clients to improve their performance, which often results in the suppression of well-performing ones. As a consequence, they usually sacrifice the system's overall average performance for achieving fairness. Motivated by this and inspired by two well-known risk modeling methods in Finance, \emph{Mean-Variance} and \emph{Mean-Semi-Variance}, we propose and study two new fair \FL algorithms, \emph{Variance Reduction} (\VRed) and \emph{Semi-Variance Reduction} (\SemiVRed). \VRed encourages equality between clients' loss functions by penalizing their variance. In contrast, \SemiVRed penalizes the discrepancy of only the worst-off clients' loss functions from the average loss. Through extensive experiments on multiple vision and language datasets, we show that, \SemiVRed achieves SoTA performance in scenarios with heterogeneous data distributions and improves both fairness and system overall average performance.
\end{abstract}

\section{Introduction}

Federated Learning \cite{McMahanMRHA17} is a framework consisting of some clients and the private data that is distributed among them, allowing training of a shared or personalized model based on the clients' data. Since the seminal work of \cite{McMahanMRHA17}, it has attracted an intensive amount of attention and much progress has been made in its different aspects, including algorithmic innovations \cite{tian_2018_fedprox,ReddiCZG20,pathak2020fedsplit,HuoYGCH20,WangYSPK20,reddi2020adaptive}, fairness \cite{McMahanMRHA17,li2019fair,MohriSS19,li2020tilted,Yue2021GIFAIRFLAA,Zhang2022ProportionalFI}, convergence analysis \cite{Khaled2019FirstAO,xiang_convergence_2019,GorbunovHR21,malek_ecmlpkdd,malekmohammadi2021operator}, personalization \cite{chen2022on,oh2022fedbabu,pmlr-v162-zhang22o,pmlr-v162-bietti22a}, and various other aspects.

Due to heterogeneity in clients' data and their resources, performance fairness is an important challenge in \FL systems. There have been some previous works addressing this problem. For instance, \cite{MohriSS19} proposed Agnostic Federated Learning (\AFL), which aims at minimizing the largest loss function among clients through a minimax optimization framework. Similarly, \cite{li2020tilted} proposed an algorithm called \TERM using tilted losses. \Ditto \cite{li_ditto_2021} is another existing algorithm based on model personalization for clients. Also, $q$-Fair Federated Learning (\qFFL) \cite{li2019fair} is an algorithm inspired by $\alpha$-fairness in wireless networks \cite{lan_axiomatic_2010}. Also, the work in \cite{Zhang2022ProportionalFI} proposed \PropFair based on Proportional Fairness, and showed that all the aforementioned fair \FL algorithms can be unified into a generalized mean framework. \GiFair \cite{Yue2021GIFAIRFLAA} achieves fairness using a different mechanism: it penalizes the discrepancy between clients' loss functions, i.e., encouraging equality of clients' losses. \texttt{FCFL} \cite{Cui2021AddressingAD} uses a constrained version of \AFL to achieve both algorithmic parity and performance consistency in \FL settings.

Being designed for fair \FL, the aforementioned algorithms usually result in the suppression of well-performing clients, due to the lower weights that the algorithms place on them or due to the equality that they enforce between clients' losses (\GiFair). As a consequence, the clients experiencing good performance with vanilla FedAvg, experience a relatively lower performance when using the above fair \FL algorithms. This is our motivation for proposing two new algorithms.

Our inspiration in this paper is a concept in Finance called \emph{risk modeling} used for portfolio selection. There are two vastly used methodologies for risk modeling: \emph{Mean-Variance} (\MV) \cite{MVsurvey,Soleimani2009MarkowitzbasedPS,meanvar_mark} and its expansion: \emph{Mean-Semi-Variance} (\MSV) \cite{msv_based_farmework,PO_downsiderisk,MSVEF,meansemivar_mark}, which are used for quantifying investment return and investment risk. Motivated by the vast usage of these methodologies and their great success in financial planning, we bring the \MV and \MSV methods to \FL by proposing \emph{Variance Reduction} (\VRed) and \emph{Semi-Variance Reduction} (\SemiVRed) algorithms, respectively. 
\section{Background}\label{sec:background}
With formal notations, we consider an \FL setting with $n$ clients for the task of multi-class classification. Let $x\in \mathcal{X}\subseteq\mathds{R}^p$ and $y \in \mathcal{Y}=\left\{1, \ldots, C \right\}$ denote the input data point and its target label, respectively. Each client $i$ has its own private data with data distribution $P_i(x,y)$. Let $h: \mathcal{X}\times\thetav\to\mathds{R}^C$ be the used predictor function, which is parameterized by $\thetav\in\mathds{R}^d$, shared among all clients. Also, let $\ell:\mathds{R}^C\times\mathcal{Y}\to \mathds{R}_+$ be the loss function, which we choose to be the cross entropy loss. Client $i$ minimizes loss function $f_i(\thetav)=\mathbb{E}_{(x,y)\sim P_i(x,y)}[\ell(h(x,\thetav), y)]$, which has minimum value $f_i^*$, on its local dataset $\mathcal{D}_i$ with size $n_i$. We denote $\frac{n_i}{N}$, where $N = \sum_i n_i$, with $p_i$.

There are various fair \FL algorithms in the literature. In \cref{tab:fairness_fl}, we have provided details of the most recent algorithms with their formulations. The existing fair \FL algorithms can be grouped into two main categories:

\begin{table*}[t]
    \caption{Objective functions  of the existing fair \FL algorithms (assuming $N_i=N_j$ for $i \neq j$).
    }
    \centering
    \begin{tabular}{ccc}
    \toprule 
  \bf FL algorithm & \bf Objective & \bf Reference \\
    \midrule
  \FedAvg  & $\sum_i f_i(\thetav)$  & \cite{McMahanMRHA17} \\
  \AFL  & $\max_i f_i (\thetav)$ & \cite{MohriSS19} \\
 \qFFL  & $\sum_i f_i^{q+1} (\thetav)$ & \cite{li2019fair} \\
  \TERM & $ \sum_i e^{\alpha f_i(\thetav)}$ & \cite{li2020tilted} \\
 \PropFair  & $-\sum_i \log (M - f_i(\thetav))$ & \cite{Zhang2022ProportionalFI}\\
 $\Delta$-\FL  & $\CVaR_{1-\alpha}$~ $(f_1(\thetav), \ldots,f_n(\thetav))$ & \cite{deltaFL}\\
 \midrule
 \GiFair  & $\sum_i f_i(\thetav) + \lambda \sum_{i\neq j} |f_i(\thetav)-f_j(\thetav)| $ & \cite{Yue2021GIFAIRFLAA}\\
 \VRed  & $\sum_i f_i(\thetav) + \beta \sum_{i} \left(f_i(\thetav) - \frac{1}{n} \sum_j f_j(\thetav)\right)^2 $ & this work\\
 
 \SemiVRed  & $\sum_i f_i(\thetav) + \beta \sum_{i} \left(f_i(\thetav) - \frac{1}{n} \sum_j f_j(\thetav)\right)_+^2 $ & this work\\

    \bottomrule
    \end{tabular}
    \label{tab:fairness_fl}
\end{table*}

\subsection{Algorithms based on generalized mean}
The first category includes \FedAvg \cite{McMahanMRHA17},
\qFFL \cite{li2019fair}, \AFL \cite{MohriSS19}, \TERM \cite{li2020tilted}, \PropFair \cite{Zhang2022ProportionalFI}. It was shown by \cite{Zhang2022ProportionalFI} that this category can be unified into a generalized mean framework \cite{Kolmogorov30}, where more attention is paid to the clients with larger losses. Also, there has been a risk measure in Finance literature, which models the one-sided nature of risks and is known as ``Conditional Value at Risk" (\CVaR). It was used by \cite{fairnessriskmeasures} for algorithmic fairness in a non-\FL setting. The work in \cite{deltaFL} used the same \CVaR risk measure as an objective function to propose $\Delta$-\FL algorithm for achieving performance fairness in \FL settings. Besides not being clear how to set the parameter $\alpha$ of the \CVaR-based objective function in \FL settings (\cref{tab:fairness_fl}), it reduces to the average of loss functions of a subset of clients with the largest loss values for $0<\alpha<1$ (see eq. 24 in \cite{fairnessriskmeasures}), i.e., assigns a zero weight to the rest of clients in the average. Hence, \CVaR's objective function is a generalization of the objective function of \AFL.

\subsection{Algorithms based on enforcing equality}
The second category of fair \FL algorithms, which includes \GiFair, is based on encouraging equality between clients' train loss values. \GiFair adds a regularization term to the objective of \FedAvg to penalize the pairwise discrepancy between clients' loss values (see \cref{tab:fairness_fl}), and enforces equality between them to achieve performance fairness.

A common feature of both the categories above is their emphasis on the clients with relatively larger losses, which usually  results in the suppression of the well-performing clients. This may result in the degradation of the overall average performance too (measured by the mean test accuracy across clients). In the next sections, we will see that \SemiVRed can achieve fairness by adding a regularization term: \emph{semi-variance} of clients' loss functions and improves both fairness \emph{and} the system overall performance simultaneously. There have been some works in the literature in a similar context of ``variance" regularization: \cite{Maurer2009,Namkoong2017} proposed regularizing the empirical risk minimization (ERM) by the empirical variance of losses across training samples to balance bias and variance and improve out-of-sample (test) performance and convergence rate. Similarly, \cite{shivaswamy2010a} proposed boosting binary classifiers based on a variance penalty applied to exponential loss. Variance regularization has also been used for out-of-distribution (domain) generalization: assuming having access to data from multiple training domains,  \cite{Krueger2021OutofDistributionGV} proposed penalizing variance of training risks across the domains as a method of distributionally robust optimization for domain generalization.

\section{Risk modeling methods in Finance: \emph{Mean-Variance} and \emph{Mean-Semi-Variance}}
\label{sec:risk_modeling}

\emph{Mean-Variance} (\MV) and \emph{Mean-Semi-Variance} (\MSV) have been two popular methods for modeling risks and gains of an investment portfolio, as the first step in financial planning.

    \textbf{\emph{Mean-Variance} (\MV)} \cite{MVsurvey,Soleimani2009MarkowitzbasedPS,meanvar_mark}.
    This method treats the return of each security in an investment portfolio as a random variable and adopts its expected value and variance to quantify the return and risk of the portfolio, respectively. An investor either minimizes the risk for a fixed expected return level or maximizes the return for a given acceptable risk level. In the former case, \MV results in the following problem:
    \begin{align}\label{eq:MV}
    &\max_{x_1, \ldots, x_n} \quad \mathbb{E}[x_1 S_1  + \ldots + x_n S_n]\\\nonumber
    & \textrm{s.t.} \quad \sigma^2[x_1 S_1  + \ldots + x_n S_n]\leq R,\quad \sum_i x_i=1,\quad x_i\geq0.
    \end{align}
    Here, $\mathbb{E}$ and $\sigma^2$ denote the expected value and variance operators, respectively. Also, $S_i$ and $x_i$ denote the random return from security $i$ and the proportion of total wealth invested in security $i$, respectively. This example provides a basic view of how \MV model works. Other closely related measures of risk in the \MV model include the standard deviation ($\sigma$) and coefficient of variation ($\sigma/\mathbb{E}$). 
    However, the \emph{Mean-Variance} modeling of risk is debatable: any uncertain return above the expectation is usually not considered as risk in the common sense, but the \MV model does so. This shortcoming is resolved by the \emph{Mean-Semi-Variance} model.
    
    \textbf{\emph{Mean-Semi-Variance} (\MSV)} \cite{msv_based_farmework,PO_downsiderisk,MSVEF,meansemivar_mark}. Having recognized the importance of the (often) one-sided nature of risks, \MSV model proposed a \emph{downside} risk measure known as \emph{semi-variance}, which we denote by $\sigma_<^2$. Unlike variance, it is only concerned with the downside of the return, i.e., only the cases that the return drops \emph{below} a predefined threshold. With this risk modeling method, problem \ref{eq:MV} changes to the following:
    \begin{align}\label{eq:MSV}
    & \max_{x_1, \ldots, x_n} \quad  \mathbb{E}[x_1 S_1  + \ldots + x_n S_n]\\ \nonumber
    & \textrm{s.t.} \quad \sigma_<^2[x_1 S_1  + \ldots + x_n S_n]\leq R, \quad \sum_i x_i=1, \quad x_i\geq0,
    \end{align}
    where the operator \emph{semi-variance} ($\sigma_<^2$) measures the \emph{downsides} of the return: $\sigma_<^2[z]=\mathbb{E}[(\mathbb{E}[z]-z)_+^2]$. \MSV is a preferable alternative for the \MV model as its modeling of the risk is more consistent with our perception from an investment risk. Again, the problem above gives a basic understanding of how the \MSV model works. More complex variations of \MV and \MSV models have been developed for complex and unpredictable financial markets \cite{MSV_MAP,MVsurvey,MSVEF}.


\section{\MV and \MSV models for fair \FL}\label{sec:MV_MSV_FL}

We now propose two fair \FL algorithms by using the \MV and \MSV to quantify the inequality between clients' utilities. Inspired by \cite{Zhang2022ProportionalFI}, we take a simple approach and define $u_i(\thetav) = M-f_i(\thetav)$ as the utility of client $i$, where $M$ can be treated as a utility baseline. The smaller the loss function of a client, the larger its utility: \emph{the utility of a client can be used to roughly represent the test accuracy of the shared model, parameterized by $\thetav$, on its local data}. With this definition, we propose to model the inequality between clients by the variance and semi-variance of their utilities, resulting in the \VRed and \SemiVRed algorithms. 

\subsection{The \VRed algorithm}
\VRed models the inequality between clients' utilities with their variance and aims to minimize the following:
\begin{align}\label{eq:v-red}
&F_{\VRed}(\thetav) = \sum_i p_i f_i(\thetav) +  \beta \sum_{i} p_i\bigg(u_i(\thetav) -  \sum_j p_j u_j(\thetav)\bigg)^2 \nonumber \\
&= \sum_i p_i f_i(\thetav) +  \beta \sum_{i} p_i\bigg(f_i(\thetav) - \sum_j p_j f_j(\thetav)\bigg)^2.
\end{align}

\begin{algorithm}[tb]
\caption{\VRed and \SemiVRed}
\label{alg:Var_Min}
\KwIn{global epoch $T$, loss functions $f_i$, number of samples $n_i$ for client $i$, number of total samples $N$, initial global model $\thetav^{(0)}$, local step number $K_i$ for client $i$, learning rate $\eta$}

Let $p_i = \frac{n_i}{N}$ for $i \in \{0, 1, \dots, n-1\}$ 

\For{$t=0, 1, \dots, T$}
{randomly select $\Sc^{(t)} \subseteq [n]$ 

$\x{t}_{i} = \thetav^{(t)}$ for $i\in \Sc^{(t)}$, $N = \sum_{i\in \Sc^{(t)}} n_i$ 

    \For(\tcp*[h]{in parallel}){$i$ in $\Sc^{(t)}$}
    {starting from $\theta_i^{(t)}$, take $K_i$ local SGD steps on $f_i(\theta_i^{(t)})$ with learning rate $\eta$ to find $\theta_i^{(t+1)}$ 

    compute $\Delta^{(t)}_i =  \theta_i^{(t)} - \theta_i^{(t+1)}$
    }
compute $\overline{f}({\thetav^{(t)}})=\sum_i p_i f_i({\thetav^{(t)}})$ and $\overline{\Delta}^{(t)} = \sum_i p_i \Delta_i^{(t)}$

\uIf {\VRed}
{compute $\Delta^{(t)} = \sum_i p_i \Delta^{(t)}_{i} + 2\beta \sum_{i} p_i(f_i({\thetav^{(t)}}) - \overline{f}({\thetav^{(t)}}))(\Delta^{(t)}_{i} - \overline{\Delta}^{(t)})$ }
\uElseIf{\SemiVRed}
{compute $\Delta^{(t)} = \sum_i p_i \Delta^{(t)}_{i} 
    + 2\beta \sum_{i} p_i(f_i({\thetav^{(t)}}) - \overline{f}({\thetav^{(t)}}))_+(\Delta^{(t)}_{i} - \overline{\Delta}^{(t)})$ }

update $\thetav^{(t+1)} = \thetav^{(t)} - \Delta^{(t)}$
}
\KwOut{global model $\thetav^{(T)}$}
\end{algorithm}

\VRed regularizes the objective function of vanilla \FedAvg with variance of clients losses (utilities). Let us derive the \VRed federated learning algorithm. By taking the gradient of \eqref{eq:v-red} and multiplying it by the step size $\eta$, we have:
\begin{align}
& \eta \nabla F_{\VRed}(\thetav) = \sum_i p_i \eta \nabla f_i(\thetav) +
\nonumber \\
& 2\beta \sum_{i} p_i\Big(f_i(\thetav) -  \sum_j p_j f_j(\thetav)\Big)\Big(\eta \nabla f_i(\thetav) -  \sum_j p_j \eta \nabla f_j(\thetav)\Big).
\end{align}
This equation immediately leads to an \FL algorithm, by replacing the gradient $\eta \nabla f_i(\thetav)$ with the pseudo-gradient (i.e., the opposite of the local update), denoted by $\Delta^{(t)}_{i}$:

\begin{align}\label{eq:grad-V-Red}
& \eta \nabla F_{\VRed}(\thetav) \nonumber \\
& = \sum_i p_i \Delta^{(t)}_{i} + 2\beta \sum_{i} p_i\Big(f_i(\thetav) - \overline{f}(\thetav)\Big)\Big(\Delta^{(t)}_{i} -   \overline{\Delta}^{(t)}\Big),
\end{align}
where $\overline{f}(\thetav)=\sum_i p_i f_i(\thetav)$ and $\overline{\Delta}^{(t)} = \sum_i p_i \Delta_i^{(t)}$. The corresponding algorithm is included in \cref{alg:Var_Min}. There is a parameter $\beta$ which tunes the effect of the regularization term, which needs to get tuned for better performance. Note that this is a new aggregation rule: instead of simply averaging the local models, it has an additional second term, which relates to the variance of clients losses. If all clients are identical (i.e., no heterogeneity across clients), this term vanishes.

\subsubsection{An interpretation of \VRed}
With the definition of utilities above ($u_i(\thetav) = M-f_i(\thetav)$), \VRed is aimed to penalize the variance of clients' utilities. As such, one potential drawback of \VRed is that it may result in the suppression of well-performing clients (the ones with small losses) for reducing the variance, which is the same drawback that \GiFair \cite{Yue2021GIFAIRFLAA} had. Hence, the final model's overall performance averaged across clients may get sacrificed. Despite this, \GiFair minimizes an upper bound of \VRed objective function: assuming $p_i = \frac{1}{n}  ~(i=1, \ldots, n)$, i.e., all clients have the same number of data points, we have the following upper bound on \VRed objective function (see \Cref{eq:relate_Gi_Ve_proof} in the appendix for derivation):

\begin{align}\label{eq:relate_Gi_Ve}
F_{\VRed}(\thetav) & = \sum_i f_i(\thetav) + \beta \sum_{i} \bigg|f_i(\thetav) - \frac{1}{n} \sum_j f_j(\thetav)\bigg|^2 \nonumber \\
& \leq \sum_i f_i(\thetav) + \frac{2\beta}{n} \sum_{j \neq i} \big| f_i(\thetav) - f_j(\thetav)\big|^2 \nonumber \\
& \leq \sum_i f_i(\thetav) + \frac{2\beta}{n} \sum_{j \neq i} \big| f_i(\thetav) - f_j(\thetav)\big| \nonumber \\
& = F_{\GiFair}(\thetav),
\end{align}
where the second inequality is true when clients' loss functions are all less than 1, which happens after some communication rounds. Therefore, \GiFair in fact minimizes an upper bound of \VRed's objective function. This relation between the two algorithms can explain why \VRed usually outperforms \GiFair in terms of fairness in our experiments.

In typical \FL settings, the global objective function can be written as a weighted sum of clients' loss functions, i.e., $F(\thetav) := \sum_{i=1}^n w_i h_i(\thetav)$, where $h_i(\thetav)$ is used by client $i$ as a surrogate of the global objective and is optimized using the client's local data. Also, the weight $w_i$ represents the importance of client $i$ loss function in the global objective function $F(\thetav)$. For example, \FedAvg simply uses $h_i(\thetav)=f_i(\thetav)$ and $w_i=p_i$ ($p_i=\frac{n_i}{N}$, see \cref{alg:Var_Min}) and \qFFL uses $h_i(\thetav)=f_i^{q+1}(\thetav)$ and $w_i=p_i$. A direct consequence of the above summation form for $F(\thetav)$ is:
\begin{align}\label{eq:global_loss_grad}
    \nabla F(\thetav) = \sum_{i=1}^n w_i \nabla h_i(\thetav).
\end{align}
Again, the weight $w_i$ represents the importance of the client $i$'s model updates. In \cref{lemma:grad_vred}, we show that the gradient of the global objective of \VRed in \eqref{eq:v-red}, can be written in the form of \eqref{eq:global_loss_grad}. For simplicity and easier interpretation, we assume $p_i = \frac{1}{n}  ~(i=1, \ldots, n)$, i.e., all clients have the same number of data points.

\begin{restatable}{lemma}{gradvred}\label{lemma:grad_vred}
Assuming equal dataset sizes for all clients, for any model parameter $\thetav$, the gradient of the global objective $F_{\VRed}(\thetav)$ defined in \eqref{eq:v-red} can be expressed as 
\begin{align}\label{eq:lin_sum_vred}
    &\nabla F_{\VRed}(\thetav) = \sum_{i=1}^n w_i(\thetav) \nabla f_i(\thetav), \nonumber \\ &w_i(\thetav)=\frac{1}{n}+\frac{2\beta(f_i(\thetav)-\overline{f}(\thetav))}{n}, \nonumber \\
    &\overline{f}(\thetav) = \frac{\sum_i f_i(\thetav)}{n}.
\end{align}

\end{restatable}

The proof is deferred to \S \ref{sec:appendix_proofs} in the appendix. \cref{lemma:grad_vred} shows that, unlike \FedAvg that would assign $w_i = \frac{1}{n},  i=1, \ldots, n$ to all clients, \VRed assigns a relatively larger weight ($w_i$) to clients with larger loss functions, and dynamically updates the weights $w_i$ at each communication round. Importantly, based on \eqref{eq:lin_sum_vred}, in order for all clients to get assigned a positive weight, the parameter $\beta$ needs to satisfy the following inequality: $0 \leq \beta < \beta_{\VRed}^{max}(\thetav) \triangleq \frac{1}{2(\overline{f}(\thetav)- \min_i \{f_i(\thetav)\})}.$

\subsection{The \SemiVRed algorithm}
Inspired by the discussion on the superiority of \MSV over \MV in \S~\ref{sec:risk_modeling} for risk modeling, we propose an extension of \VRed. Consider the following objective function instead of \eqref{eq:v-red}:
\begin{align}\label{eq:semivar-red}
&F_{\texttt{SVRed}}(\thetav) = \sum_i p_i f_i(\thetav) +  \beta \sum_{i} p_i\bigg(f_i(\thetav) - \sum_j p_j f_j(\thetav)\bigg)_+^2,
\end{align}

where $\sigma^2_<$ denotes the semi-variance of clients' utilities. This objective, in addition to minimizing the mean loss, \emph{reduces the semi-variance of clients' losses}, meaning that only those clients that have a relatively small utility $u_i(\thetav)$ (or equivalently a large loss $f_i(\thetav)$) contribute to the semi-variance  regularization term in  \cref{eq:semivar-red}. Similar to what we did for \VRed, if we take the gradient of \eqref{eq:semivar-red}, we have:
\begin{align}\label{eq:grad-semivred}
& \eta \nabla F_{\texttt{SVRed}}(\thetav) \nonumber \\
&= \sum_i p_i \Delta^{(t)}_{i} + 2\beta \sum_{i} p_i\Big(f_i(\thetav) - \overline{f}(\thetav)\Big)_+\Big(\Delta^{(t)}_{i} -  \overline{\Delta}^{(t)}\Big),
\end{align}
where $\Delta^{(t)}_{i}$ is the pseudo-gradient (i.e., the opposite of the local update) of user $i$.
The corresponding algorithm is included in \cref{alg:Var_Min}. Again, we have a tunable parameter $\beta$ which sets the effect of the regularization term and needs to get tuned for better performance.

\subsection{Can we interpret what \SemiVRed does?}
\subsubsection{Optimization aspect} 
We will show in \cref{lemma:grad_semivred} that, in contrast to \VRed (and \GiFair) and thanks to its more efficient regularization, \SemiVRed does not suppress the well-performing clients to help the worst-off ones. Again, for simplicity and easier interpretation, we assume equal dataset sizes for all clients, which leads to $p_i = \frac{1}{n} ~ (i=1, \ldots, n$).

\begin{restatable}{lemma}{gradsemivred}\label{lemma:grad_semivred}
In each communication round between the clients and the server, let $>_C$ denote the set of clients whose local loss function is greater than the average loss function $\overline{f}(\thetav)$. Assuming equal dataset sizes for all clients, for any model parameter $\thetav$, the gradient of the global objective $F_{\texttt{SVRed}}(\thetav)$ defined in \eqref{eq:semivar-red} can be expressed as
\begin{align}\label{eq:lin_sum_semivred}
    \nabla F_{\texttt{SVRed}}(\thetav) = \sum_{i=1}^n w_i(\thetav) \nabla f_i(\thetav),
\end{align}

where $\overline{f}(\thetav) = \frac{\sum_i f_i(\thetav)}{n}$ and: 

\scriptsize
\begin{equation}\label{eq:semivred_weights}
    w_i(\thetav)  =
\begin{dcases}
    \frac{1}{n}+\frac{2\beta(f_i(\thetav)-\overline{f}(\thetav))}{n} - \frac{2\beta\sum_{j\in >_{\mathcal{C}
}}(f_j(\thetav)-\overline{f}(\thetav))}{n^2} \text{, if $i \in >_{\mathcal{C}}$ } \\
 \frac{1}{n}-\frac{2\beta\sum_{j\in >_{\mathcal{C}
}}(f_j(\thetav)-\overline{f}(\thetav))}{n^2} \text{, if $i \notin >_{\mathcal{C}}$ }
\end{dcases}
\end{equation}

\end{restatable}

Similar to \VRed, there is an upper-bound for $\beta$ to ensure positive weights for all clients in \eqref{eq:semivred_weights}: \\
$0 \leq \beta < \beta_{\texttt{SVRed}}^{max}(\thetav) \triangleq \frac{n}{2\sum_{j\in >_{\mathcal{C}
}}(f_j(\thetav)-\overline{f}(\thetav))} $.

\begin{remark}{}\label{remark:comparison} Interesting points can be observed by comparing \cref{lemma:grad_vred} and \cref{lemma:grad_semivred}. First, both of the algorithms pay more attention to worst-off clients by assigning larger weights to their gradients. However, \SemiVRed assigns relatively larger weights to the well-performing clients. Also, for \VRed, $w_i(\thetav)=\frac{1}{n}+\frac{2\beta(f_i(\thetav)-\overline{f}(\thetav))}{n}$. So the better a client performs, the more it is suppressed by the algorithm. In contrast \SemiVRed assigns weights to well-performing clients depending on how bad the worst-off clients perform compared to the average performance. As the performance of worst-off clients improves gradually, the algorithm also lets the well-performing ones for further improvement (instead of strictly suppressing them like \VRed).
\end{remark}

\subsubsection{Scenarios with large label shifts}
We now provide another interesting interpretation of \SemiVRed, related to data heterogeneity in \FL. In order for an easier interpretability and understanding, we assume $P_i(x,y)=P_i(x|y)P_i(y)=P(x|y)P_i(y)$. This means that the class conditional distribution of input $x$ is identical for all clients, and there is no concept shift across them. However, there is label shift across clients. Having made this assumption, we define $\overline{\ell}_j(\thetav)=\mathbb{E}_{x\sim P(x|y=j)}[\ell(h(x,\thetav), j)]$ as the average loss of predictor $h$ on class $j$. Using this notation, we have \cref{lemma:loss_label_shift} about the objective function \eqref{eq:semivar-red} of \SemiVRed.

\begin{restatable}{lemma}{losslabelshift}\label{lemma:loss_label_shift}
Assuming $P_i(x,y)=P_i(x|y)P_i(y)=P(x|y)P_i(y)$ for $i \in \{1, \ldots, n\}$, for any parameter $\thetav$, \SemiVRed global objective \eqref{eq:semivar-red} can be expressed as 
\begin{align}\label{eq:labelshift_sum_semivred}
    &F_{\texttt{SVRed}}(\thetav) \nonumber \\
    & = \sum_{j=1}^C \overline{P}(j)\overline{\ell}_j(\thetav) + \frac{\beta}{n} \sum_{i=1}^n \Big(\sum_{j=1}^C [P_i(j) - \overline{P}(j)]\overline{\ell}_j(\thetav)\Big)_+^2, 
\end{align}
where $\overline{P}(j) = \frac{\sum_{i=1}^n P_i(j)}{n}$ is the marginal distribution of class $j$ in the global dataset. 
\end{restatable}

\begin{proof}
From \eqref{eq:semivar-red} and with $p_i=\frac{1}{n}$, we have:

\begin{align}\label{eq:fedavg_simplified}
\overline{f}(\thetav) &= \sum_{i=1}^n \frac{f_i(\thetav)}{n} = \frac{1}{n}\sum_{i=1}^n \Big[\mathbb{E}_{(x,y)\sim p_i(x,y)}[\ell(h(x,\thetav), y)]\Big]  \nonumber  \\
&= \frac{1}{n}\sum_{i=1}^n \Big[\sum_{j=1}^C p_i(j) \times \mathbb{E}_{(x,y)\sim p(x|y=j)}[\ell(h(x,\thetav), j)]\Big]
\nonumber  \\
&= \frac{1}{n}\sum_{i=1}^n \Big[\sum_{j=1}^C p_i(j) \overline{\ell}_j(\thetav)]\Big] = \sum_{j=1}^C \Big[ (\frac{\sum_{i=1}^n p_i(j)}{n}) \overline{\ell}_j(\thetav)\Big] 
\nonumber  \\
&= \sum_{j=1}^C \overline{p}(j)\overline{\ell}_j(\thetav).
\end{align}
Similarly, we can rewrite the clients' local loss functions, and we get to:
\begin{align}
f_i(\thetav) = \sum_{j=1}^C p_i(j)\overline{\ell}_j(\thetav).
\end{align}
By plugging in the above equivalences for $f_i(\thetav)$ and $\overline{f}(\thetav)$ into \eqref{eq:semivar-red}, we get to \eqref{eq:labelshift_sum_semivred}.
\end{proof}

Note that $P_i(j)$ and $\overline{P}(j)$ show the ratio of class $j$ in client $i$'s local dataset and the global dataset, respectively. Based on \eqref{eq:labelshift_sum_semivred}, \SemiVRed is capable of improving the performance of the predictor $h$ in extreme class imbalance scenarios: consider when a label $j$ is over-represented in a client $i$'s data (i.e., $P_i(j)\approx 1$) and under-represented in the global data (i.e., $\overline{P}(j)\approx 0$). In that case, the regularization term up-weights the class $j$ in the global objective function, hence improving the performance of client $i$, which was mostly holding samples with label $j$. For better understanding of this, lets see \cref{exmp:semivred_label_shift} in the following, which we have borrowed from \cite{CLIMB2022}, and shows that \SemiVRed can handle scenarios with large class imbalance efficiently.

\begin{restatable}{example}{semivred_label_shift}\label{exmp:semivred_label_shift}
Let $u$ be the uniform distribution over the existing $C$ classes. Also, let $\delta_c$ be the Dirac distribution of class $c$. Also, lets assume that $C=2$ (binary classification problem). For the $n$ existing clients, we have: 

\begin{equation}\label{eq:exmp_dist}
    p_i(y)  =
\begin{dcases}
    \alpha u + (1-\alpha)\delta_1 \text{~~~~if $i=1$ } \\
    \alpha u + (1-\alpha)\delta_2\text{~~~~if $i \in \left\{2, \ldots, n\right\}$ }
\end{dcases}
\end{equation}

Accordingly, we have:

\begin{equation}\label{eq:exmp_p1}
    p_i(1)  =
\begin{dcases}
    1-\frac{\alpha}{2} \text{~~~~if $i=1$ } \\
    \frac{\alpha}{2} \text{~~~~~~~~~~~if $i \in \left\{2, \ldots, n\right\}$ } 
\end{dcases}
\end{equation}

\begin{equation}\label{eq:exmp_p2}
    p_i(2)  =
\begin{dcases}
    \frac{\alpha}{2} \text{~~~~~~~~~~~if $i=1$ } \\
    1-\frac{\alpha}{2} \text{~~~~if $i \in \left\{2, \ldots, n\right\}$ } 
\end{dcases}
\end{equation}

Therefore, we have:

\begin{equation}\label{eq:exmp_losses}
    f_i(\thetav)  =
\begin{dcases}
    (1-\frac{\alpha}{2})\overline{\ell}_1(\thetav) + \frac{\alpha}{2}\overline{\ell}_2(\thetav) \text{~~~~if $i=1$ } \\
    \frac{\alpha}{2}\overline{\ell}_1(\thetav) + (1-\frac{\alpha}{2})\overline{\ell}_2(\thetav) \text{~~~~if $i \in \left\{2, \ldots, n\right\}$ } 
\end{dcases}
\end{equation}

Hence, according to \cref{eq:fedavg_simplified}, we can rewrite the objective function of \FedAvg as:
\begin{equation}\label{eq:exmp_meanloss}
    \overline{f}(\thetav) = \Big(\frac{\alpha}{2} + \frac{1-\alpha}{n}\Big)\overline{\ell}_1(\thetav) + \Big(\frac{\alpha}{2}+\frac{(1-\alpha)(n-1)}{n}\Big)\overline{\ell}_2(\thetav)
\end{equation}

Clearly, we can see that if $\alpha\approx0$ and $n$ is large, 
then $\overline{\ell}_1(\thetav)$, which is the loss over the minority data will have a small weight, which leads to $\overline{\ell}_1(\thetav)$ being larger than $\overline{\ell}_2(\thetav)$ and poor performance on the minority class $1$, and client $i=1$. Now, if we rewrite \SemiVRed objective function from \cref{eq:labelshift_sum_semivred}, we have:

\begin{align}\label{eq:exmp_semivred_loss}
    & F_{\texttt{SVRed}}(\thetav) \nonumber \\
    &= \Big(\frac{\alpha}{2} + \frac{1-\alpha}{n}\Big)\overline{\ell}_1(\thetav) + \Big(\frac{\alpha}{2}+\frac{(1-\alpha)(n-1)}{n}\Big)\overline{\ell}_2(\thetav) \nonumber \\
    & + \frac{\beta(n-1)^2(1-\alpha)^2}{n^3}\Big(\overline{\ell}_1(\thetav)-\overline{\ell}_2(\thetav)\Big)^2 \nonumber \\
    & = \Bar{f}(\thetav) + \frac{\beta(n-1)^2(1-\alpha)^2}{n^3}\Big(\overline{\ell}_1(\thetav)-\overline{\ell}_2(\thetav)\Big)^2.
\end{align}

The extra regularization term improves $\overline{\ell}_1(\thetav)$ compared to vanilla \FedAvg,. Hence, the performance of client $1$ and consequently, fairness in the system improves.
\end{restatable}

\begin{table*}[t]
\centering
\caption{Comparison between the performance of different algorithms on CIFAR-100. \textbf{Second column:} the percentage (\%) of suffering clients with improved test accuracy compared to \FedAvg. The value in parentheses shows the amount of test accuracy improvement averaged over suffering clients. \textbf{Third column:} the percentage (\%) of well-performing clients with degraded test accuracy compared to \FedAvg. The value in parentheses shows the amount of test accuracy improvement averaged over well-performing clients. \textbf{Fourth column:} the amount of improvement in the overall mean test accuracy compared to \FedAvg.}
\label{tab:comparison_vred_semivred}
\small
\setlength\tabcolsep{2pt}
\begin{tabular}{ccccc}
\toprule
\bf{Algorithm} & \bf{Improved suffering clients} & \bf{Degraded well-performing clients} & \bf{Overall accuracy improvement} 
\\ 
\midrule
\qFFL & 52.08$_{\pm 11.95}$ (+0.69)&  54.21$_{\pm 15.38}$ (-0.79) &  +0.04$_{\pm 0.41}$
\\

\AFL & 51.18$_{\pm 9.26}$ (+0.56) &  77.25$_{\pm 12.85}$ (-3.50)
&  -1.22$_{\pm 0.95}$
\\

\GiFair & 60.55$_{\pm 6.17}$ (+0.86)&  68.22$_{\pm 16.13}$ (-2.03) &  -0.40$_{\pm 0.61}$
\\

\TERM & 23.66$_{\pm 5.34}$ (-1.12) & 87.93$_{\pm 1.81}$ (-3.57) &  -2.20$_{\pm 0.66}$
\\

\PropFair & 8.33$_{\pm 1.69}$ (-4.05)&  92.41$_{\pm 4.28}$ (-6.74) &  -5.23$_{\pm 0.96}$
\\
$\Delta$-\FL & 46.55$_{\pm 6.75}$ (+0.14)&  78.75$_{\pm 2.32}$ (-4.31) &  -1.81$_{\pm 0.25}$
\\
\midrule
\VRed & 60.50$_{\pm 12.52}$ (+1.11)&  60.62$_{\pm 7.53}$ (-0.94) &  +0.21$_{\pm 0.06}$
\\

\SemiVRed & \bf 65.40$_{\pm 6.29}$ \textbf{(+1.47)}&  \bf 53.17$_{\pm 6.75}$ (-0.42) &  \bf+0.64$_{\pm 0.30}$
\\
\bottomrule
\end{tabular}
\vspace{-1em}
\end{table*}

\pgfplotstableread[row sep=\\, col sep=&]
{alg & mean & worst20 \\
\rotatebox[origin=c]{90}{\scriptsize FedAvg} & 43.45 & 18.86 \\
\rotatebox[origin=c]{90}{\scriptsize $q$-FFL} & 45.46 & 21.23 \\
\rotatebox[origin=c]{90}{\scriptsize AFL} & 0 & 0 \\
\rotatebox[origin=c]{90}{\scriptsize GiFair} & 45.05 & 22.65 \\
\rotatebox[origin=c]{90}{\scriptsize TERM} & 45.61 & 24.89 \\
\rotatebox[origin=c]{90}{\scriptsize PropFair} & 36.95 & 12.49 \\
\rotatebox[origin=c]{90}{\scriptsize $\Delta$-FL} & 40.32 & 16.94 \\
\rotatebox[origin=c]{90}{\scriptsize VRed} & 44.43 & 24.28 \\
\rotatebox[origin=c]{90}{\scriptsize Semi-VRed} & 45.47 & 27.08 \\
}\cifartendata

\pgfplotstableread[row sep=\\, col sep=&]
{alg & mean & worst20 \\
\rotatebox[origin=c]{90}{\scriptsize FedAvg} & 20.20 & 11.07 \\
\rotatebox[origin=c]{90}{\scriptsize $q$-FFL} & 20.25 & 11.09 \\
\rotatebox[origin=c]{90}{\scriptsize AFL} & 18.98 & 11.31 \\
\rotatebox[origin=c]{90}{\scriptsize GiFair} & 19.81 & 11.19 \\
\rotatebox[origin=c]{90}{\scriptsize TERM} & 18.00 & 10.02 \\
\rotatebox[origin=c]{90}{\scriptsize PropFair} & 14.97 & 7.00 \\
\rotatebox[origin=c]{90}{\scriptsize $\Delta$-FL} & 18.39 & 10.06 \\
\rotatebox[origin=c]{90}{\scriptsize VRed} & 20.42 & 11.21 \\
\rotatebox[origin=c]{90}{\scriptsize Semi-VRed} & 20.85 & 11.86 \\
}\cifarhundreddata

\pgfplotstableread[row sep=\\, col sep=&]
{alg & mean & worst20 \\
\rotatebox[origin=c]{90}{\scriptsize FedAvg} & 86.26 & 56.87 \\
\rotatebox[origin=c]{90}{\scriptsize $q$-FFL} & 86.63 & 57.77 \\
\rotatebox[origin=c]{90}{\scriptsize AFL} & 86.45 & 57.58 \\
\rotatebox[origin=c]{90}{\scriptsize GiFair} & 86.28 & 56.97 \\
\rotatebox[origin=c]{90}{\scriptsize TERM} & 86.34 & 57.21 \\
\rotatebox[origin=c]{90}{\scriptsize PropFair} & 86.01 & 56.53 \\
\rotatebox[origin=c]{90}{\scriptsize $\Delta$-FL} & 86.11 & 57.10 \\
\rotatebox[origin=c]{90}{\scriptsize VRed} & 85.79 & 57.66 \\
\rotatebox[origin=c]{90}{\scriptsize Semi-VRed} & 85.83 & 58.00 \\
}\cinictendata

\pgfplotstableread[row sep=\\, col sep=&]
{alg & mean & worst20 \\
\rotatebox[origin=c]{90}{\scriptsize FedAvg} & 40.34 & 27.12 \\
\rotatebox[origin=c]{90}{\scriptsize $q$-FFL} & 37.79 & 24.12 \\
\rotatebox[origin=c]{90}{\scriptsize AFL} & 0 & 0 \\
\rotatebox[origin=c]{90}{\scriptsize GiFair} & 40.34 & 27.12 \\
\rotatebox[origin=c]{90}{\scriptsize TERM} & 40.34 & 27.10 \\
\rotatebox[origin=c]{90}{\scriptsize PropFair} & 41.76 &28.75 \\
\rotatebox[origin=c]{90}{\scriptsize $\Delta$-FL} & 39.94 & 26.94 \\
\rotatebox[origin=c]{90}{\scriptsize VRed} & 42.90 & 30.39 \\
\rotatebox[origin=c]{90}{\scriptsize Semi-VRed} & 42.90 & 30.34 \\
}\stackoverflowdata

\begin{figure*}[hbt!]

\subfigure{
\begin{tikzpicture}
\begin{axis}
[title= \scriptsize CIFAR-10,
width  = 0.45\textwidth,
height = 4.2cm,
bar width=5pt,
ymajorgrids = true,
ylabel={\scriptsize test accuracy},
legend style={
at={(0.985,0.74)},
anchor=south east,
column sep=1ex
},
ymin=10, 
ymax=60,
ybar,
symbolic x coords = {\rotatebox[origin=c]{90}{\scriptsize FedAvg}, \rotatebox[origin=c]{90}{\scriptsize $q$-FFL}, \rotatebox[origin=c]{90}{\scriptsize AFL}, \rotatebox[origin=c]{90}{\scriptsize GiFair}, \rotatebox[origin=c]{90}{\scriptsize TERM}, \rotatebox[origin=c]{90}{\scriptsize PropFair}, \rotatebox[origin=c]{90}{\scriptsize $\Delta$-FL}, \rotatebox[origin=c]{90}{\scriptsize VRed}, \rotatebox[origin=c]{90}{\scriptsize Semi-VRed}},
xtick=data,
]

\addplot table[x=alg, y=mean]{\cifartendata};
\addplot table[x=alg, y=worst20]{\cifartendata};
\addplot[blue,sharp plot,dashed]
coordinates {(\rotatebox[origin=c]{90}{\scriptsize FedAvg},45.61) (\rotatebox[origin=c]{90}{\scriptsize Semi-VRed},45.61)};
\addplot[red,sharp plot,dashed]
coordinates {(\rotatebox[origin=c]{90}{\scriptsize FedAvg},27.08) (\rotatebox[origin=c]{90}{\scriptsize Semi-VRed},27.08)};
\legend{\scriptsize mean test accuracy, \scriptsize worst 10\% test accuracy}
\end{axis}

\end{tikzpicture}
}
\hfill
\subfigure{
\begin{tikzpicture}
\begin{axis}
[title= \scriptsize CIFAR-100,
width  = 0.45\textwidth,
height = 4.2cm,
bar width=5pt,
ymajorgrids = true,
ymin=5, 
ybar,
symbolic x coords = {\rotatebox[origin=c]{90}{\scriptsize FedAvg}, \rotatebox[origin=c]{90}{\scriptsize $q$-FFL}, \rotatebox[origin=c]{90}{\scriptsize AFL}, \rotatebox[origin=c]{90}{\scriptsize GiFair}, \rotatebox[origin=c]{90}{\scriptsize TERM}, \rotatebox[origin=c]{90}{\scriptsize PropFair}, \rotatebox[origin=c]{90}{\scriptsize $\Delta$-FL}, \rotatebox[origin=c]{90}{\scriptsize VRed}, \rotatebox[origin=c]{90}{\scriptsize Semi-VRed}},
xtick=data,
]

\addplot table[x=alg, y=mean]{\cifarhundreddata};
\addplot table[x=alg, y=worst20]{\cifarhundreddata};
\addplot[blue,sharp plot,dashed]
coordinates {(\rotatebox[origin=c]{90}{\scriptsize FedAvg},20.85) (\rotatebox[origin=c]{90}{\scriptsize Semi-VRed},20.85)};
\addplot[red,sharp plot,dashed]
coordinates {(\rotatebox[origin=c]{90}{\scriptsize FedAvg},11.86) (\rotatebox[origin=c]{90}{\scriptsize Semi-VRed},11.86)};
\end{axis}
\end{tikzpicture}
}

\subfigure{
\begin{tikzpicture}
\begin{axis}
[title= \scriptsize CINIC-10,
width  = 0.45\textwidth,
height = 4.2cm,
bar width=5pt,
ymajorgrids = true,
ylabel={\scriptsize test accuracy},
legend style={
at={(0.985,0.78)},
anchor=south east,
column sep=1ex
},
ymin=50, 
ybar,
symbolic x coords = {\rotatebox[origin=c]{90}{\scriptsize FedAvg}, \rotatebox[origin=c]{90}{\scriptsize $q$-FFL}, \rotatebox[origin=c]{90}{\scriptsize AFL}, \rotatebox[origin=c]{90}{\scriptsize GiFair}, \rotatebox[origin=c]{90}{\scriptsize TERM}, \rotatebox[origin=c]{90}{\scriptsize PropFair}, \rotatebox[origin=c]{90}{\scriptsize $\Delta$-FL}, \rotatebox[origin=c]{90}{\scriptsize VRed}, \rotatebox[origin=c]{90}{\scriptsize Semi-VRed}},
xtick=data,
]

\addplot table[x=alg, y=mean]{\cinictendata};
\addplot table[x=alg, y=worst20]{\cinictendata};
\addplot[blue,sharp plot,dashed]
coordinates {(\rotatebox[origin=c]{90}{\scriptsize FedAvg},86.63) (\rotatebox[origin=c]{90}{\scriptsize Semi-VRed},86.63)};
\addplot[red,sharp plot,dashed]
coordinates {(\rotatebox[origin=c]{90}{\scriptsize FedAvg},58.00) (\rotatebox[origin=c]{90}{\scriptsize Semi-VRed},58.00)};
\end{axis}
\end{tikzpicture}
}
\hfill
\subfigure{
\begin{tikzpicture}
\begin{axis}
[title= \scriptsize StackOverflow,
width  = 0.45\textwidth,
height = 4.2cm,
bar width=5pt,
ymajorgrids = true,
legend style={
at={(0.985,0.78)},
anchor=south east,
column sep=1ex
},
ymin=20, 
ybar,
symbolic x coords = {\rotatebox[origin=c]{90}{\scriptsize FedAvg}, \rotatebox[origin=c]{90}{\scriptsize $q$-FFL}, \rotatebox[origin=c]{90}{\scriptsize AFL}, \rotatebox[origin=c]{90}{\scriptsize GiFair}, \rotatebox[origin=c]{90}{\scriptsize TERM}, \rotatebox[origin=c]{90}{\scriptsize PropFair}, \rotatebox[origin=c]{90}{\scriptsize $\Delta$-FL}, \rotatebox[origin=c]{90}{\scriptsize VRed}, \rotatebox[origin=c]{90}{\scriptsize Semi-VRed}},
xtick=data,
]

\addplot table[x=alg, y=mean]{\stackoverflowdata};
\addplot table[x=alg, y=worst20]{\stackoverflowdata};
\addplot[blue,sharp plot,dashed]
coordinates {(\rotatebox[origin=c]{90}{\scriptsize FedAvg},42.90) (\rotatebox[origin=c]{90}{\scriptsize Semi-VRed},42.90)};
\addplot[red,sharp plot,dashed]
coordinates {(\rotatebox[origin=c]{90}{\scriptsize FedAvg},30.34) (\rotatebox[origin=c]{90}{\scriptsize Semi-VRed},30.34)};
\end{axis}
\end{tikzpicture}
}

\vspace{-1em}\caption{Average and worst 10\% test accuracies. \textbf{top left:} CIFAR-10, \textbf{top right:} CIFAR-100, \textbf{bottom left:} CINIC-10, \textbf{bottom right:} StackOverflow. Due to divergence on highly heterogeneous data, results for \AFL on CIFAR-10 and StackOverFlow are not shown. All subfigures share the same legends and axis labels.}
\label{fig:comparison}
\vspace{-1em}
\end{figure*}
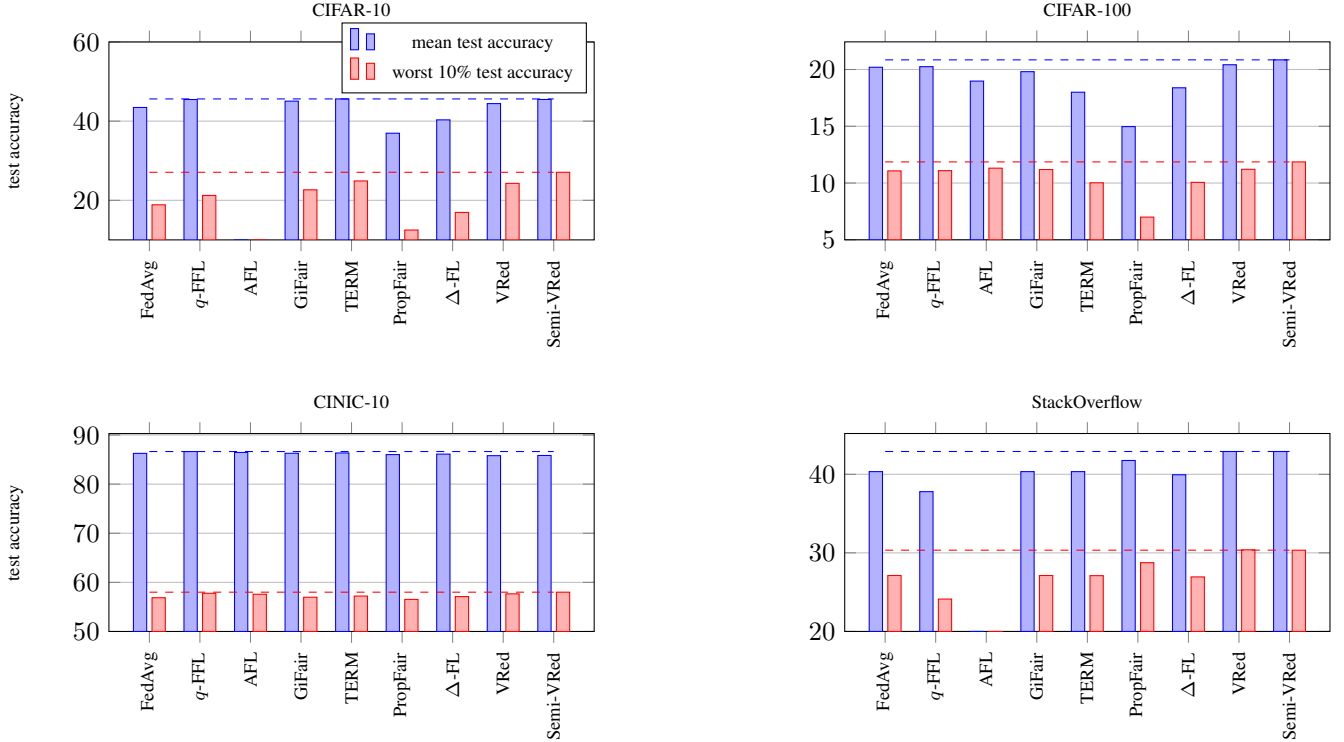

Furthermore, we have explained the relation between \VRed/ \SemiVRed and distributionally robust optimization (DRO) in \cref{vr_and_dro} to provide a better understanding of these algorithms.

\section{Experiments}\label{sec:exps}
In this section, we evaluate our proposed fair \FL algorithms with the existing algorithms in the literature.

\subsection{Experimental setup}
In the following, we explain the datasets, models and their hyperparameters as well as the metrics we use to evaluate our algorithms. For further details, see \S \ref{sec:appendix_setup} in the appendix.

\textbf{Datasets} We use four benchmark datasets existing in the literature. The datasets we use include: CIFAR-10/100 \cite{krizhevsky2009learning}, CINIC-10 \cite{darlow2018cinic} (tasks of image classification) and StackOverflow (task of next word prediction). In order to split the data among clients with high heterogeneity, we use Dirichlet distribution \cite{wang_federated_2019} with a small parameter. StackOverflow has a default realistic partition for each client. We follow the same default data distribution in our experiments.

\textbf{Models, optimizers and loss functions}
For the CIFAR-10/100 and CINIC-10 datasets, we use ResNet-18 \cite{he2016deep}. 
For the language dataset (StackOverflow), we use LSTMs \cite{hochreiter1997long}. In order to optimize the models parameters, we use SGD for minimizing average cross entropy loss. For further details, see \S \ref{sec:appendix_setup} in the appendix.

\textbf{Baseline algorithms}
We compare our \VRed and \SemiVRed algorithms with various fair \FL algorithm existing in the literature including: \FedAvg \cite{McMahanMRHA17}, \AFL \cite{MohriSS19}, \qFFL \cite{li2019fair}
, \PropFair \cite{Zhang2022ProportionalFI}, \TERM \cite{li2020tilted}, \GiFair \cite{Yue2021GIFAIRFLAA} and $\Delta$-\FL \cite{deltaFL}, which are all from the two categories of the fair \FL algorithms mentioned before.

\color{black}
\textbf{Other hyperparameters}
We implement an \FL setting where different clients participate in all communication rounds with one local epoch at each round. We use 200 communication rounds for all algorithms on the datasets to ensure their complete convergence. For CIFAR-10/100 and CINIC-10, we partition the data into 50 clients and for language datasets, we partition the data into 20 clients.

\textbf{Evaluation metrics}
the goal of proposing our novel algorithms was to achieve fairness without compromising the system overall average performance. We measure the overall performance with the \emph{mean test accuracy} across clients. In order to measure the fairness in the system, we use the worst 10\% test accuracies among clients, which is a standard metric for fairness in \FL \cite{li2020tilted,li2019fair}. In the appendix, we also use other common metrics in the literature for measuring fairness, e.g. the standard deviation of test accuracies (see \cref{tab:detailed_results} in appendix \ref{sec:appendix_setup}).



\pgfplotstableread[row sep=\\, col sep=&]
{alg & worst10 \\
\rotatebox[origin=c]{90}{FedAvg} & 23.77 \\
\rotatebox[origin=c]{90}{$q$-FFL} & 25.95 \\
\rotatebox[origin=c]{90}{AFL} & 0 \\
\rotatebox[origin=c]{90}{GiFair} & 26.52 \\
\rotatebox[origin=c]{90}{TERM} & 29.34 \\
\rotatebox[origin=c]{90}{PropFair} & 16.66 \\
\rotatebox[origin=c]{90}{$\Delta$-\FL} & 21.31\\
\rotatebox[origin=c]{90}{VRed} & 27.46 \\
\rotatebox[origin=c]{90}{Semi-VRed} & 30.34 \\
}\cifartenworsttendata

\pgfplotstableread[row sep=\\, col sep=&]
{alg & worst10 \\
\rotatebox[origin=c]{90}{FedAvg} & 12.49 \\
\rotatebox[origin=c]{90}{$q$-FFL} & 12.52 \\
\rotatebox[origin=c]{90}{AFL} & 12.72 \\
\rotatebox[origin=c]{90}{GiFair} & 12.59 \\
\rotatebox[origin=c]{90}{TERM} & 11.04 \\
\rotatebox[origin=c]{90}{PropFair} & 8.06 \\
\rotatebox[origin=c]{90}{$\Delta$-\FL} & 11.28\\
\rotatebox[origin=c]{90}{VRed} & 12.81 \\
\rotatebox[origin=c]{90}{Semi-VRed} & 13.46 \\
}\cifarhundredworsttendata

\pgfplotstableread[row sep=\\, col sep=&]
{alg & worst10 \\
\rotatebox[origin=c]{90}{FedAvg} & 62.78 \\
\rotatebox[origin=c]{90}{$q$-FFL} & 63.62 \\
\rotatebox[origin=c]{90}{AFL} & 63.04 \\
\rotatebox[origin=c]{90}{GiFair} & 62.74 \\
\rotatebox[origin=c]{90}{TERM} & 62.98 \\
\rotatebox[origin=c]{90}{PropFair} & 62.27 \\
\rotatebox[origin=c]{90}{$\Delta$-\FL} & 62.45\\
\rotatebox[origin=c]{90}{VRed} & 62.75 \\
\rotatebox[origin=c]{90}{Semi-VRed} & 62.70 \\
}\cinictenworsttendata

\pgfplotstableread[row sep=\\, col sep=&]
{alg & worst10 \\
\rotatebox[origin=c]{90}{FedAvg} & 30.35 \\
\rotatebox[origin=c]{90}{$q$-FFL} & 27.14 \\
\rotatebox[origin=c]{90}{AFL} & 0 \\
\rotatebox[origin=c]{90}{GiFair} & 30.41 \\
\rotatebox[origin=c]{90}{TERM} & 30.34 \\
\rotatebox[origin=c]{90}{PropFair} & 32.14 \\
\rotatebox[origin=c]{90}{$\Delta$-\FL} & 29.95\\
\rotatebox[origin=c]{90}{VRed} & 33.55 \\
\rotatebox[origin=c]{90}{Semi-VRed} & 35.55 \\
}\stackoverflowworsttendata

\subsection{Comparison of \VRed and \SemiVRed with other baseline algorithms}
As observed in \cref{fig:comparison}, \SemiVRed outperforms almost all the existing baseline algorithms in terms of the fairness in the system.  Furthermore, \SemiVRed improves the system's overall average performance (mean test accuracy) for three of the datasets as well. For instance, as can be observed from the results obtained for StackOverflow (see \cref{tab:detailed_results} in \S~\ref{sec:appendix_setup} in the appendix for evaluation in terms of various fairness metrics), \SemiVRed improves both fairness and mean test accuracy by 3\% and 2.7\%, respectively. Also, we can observe the competitive performance of \VRed.

\subsection{Superiority of \SemiVRed over \VRed and the other baseline algorithms}

As discussed in \S~\ref{sec:background}, the existing fair \FL\ algorithms usually suppress the well-performing clients in order to improve the clients with worse performance. However, \SemiVRed, thanks to its efficient formulation, tries to avoid this. In order to get a better understanding of this, after running the simple vanilla  \FedAvg on CIFAR-100, we divide the existing 50 clients into two sets: 1. \emph{suffering clients}: those with test accuracies below the \FedAvg mean accuracy (22 clients in our experiment) 2. \emph{well-performing clients}: those with test accuracies above the \FedAvg mean accuracy (28 clients). Then, we run each of the other algorithms and compare the performance change that they make for the two sets of clients with each other. In \cref{tab:comparison_vred_semivred}, we have done this comparison between different algorithms. The results clearly deliver two important messages: 1. the existing algorithms mostly suppress the well-performing clients, due to the more attention that they pay to the worst-off clients or enforcing equality between clients' losses 2. \SemiVRed has the least suppression of well-performing clients: with \SemiVRed, 53.17\% of the well-performing clients experience lower test accuracy compared to when using \FedAvg, and the average amount of accuracy drop for them is -0.42. \SemiVRed also results in the highest average improvement of suffering clients: 65.40\% of the suffering clients experience a higher test accuracy compared to when using \FedAvg. The average amount of accuracy improvement among suffering clients is +1.47. The above results altogether result in improvement of both the fairness and the system overall average performance simultaneously.

\section{Conclusion}
In this work, we introduced two novel fair \FL algorithms: \VRed and \SemiVRed. In order to address the drawback of most of the existing fair \FL algorithms, which is suppression of well-performing clients, we proposed \SemiVRed, which uses a one-sided regularization term to efficiently model performance unfairness in a federated learning system. Our experimental results show that \SemiVRed improves the worst-off clients performance without much suppression of well-performing ones. These altogether improve the system's overall average performance as well. Accordingly, \SemiVRed achieves SoTA performance in terms of both the overall average accuracy and fairness.

\appendix

\bibliographystyle{named}
\bibliography{ijcai23}

\appendix
\clearpage
\appendix

\onecolumn

\newpage
\begin{center}
\Large
\bf
Appendix for \emph{Semi-Variance Reduction for Fair Federated Learning}
\end{center}

\section{Proofs and derivations}\label{sec:appendix_proofs}

\paragraph{Derivation of \eqref{eq:relate_Gi_Ve}}

\begin{align}\label{eq:relate_Gi_Ve_proof}
F_{\VRed}(\thetav) = \sum_i f_i(\thetav) + \beta \sum_{i} \bigg|f_i(\thetav) - \frac{1}{n} \sum_j f_j(\thetav)\bigg|^2 &= \sum_i f_i(\thetav) + \beta \sum_{i} \bigg|\frac{n-1}{n} f_i(\thetav) - \frac{1}{n} \sum_{j\neq i} f_j(\thetav)\bigg|^2 \nonumber \\
& = \sum_i f_i(\thetav) + \frac{\beta}{n^2}\sum_{i}\bigg|\sum_{j\neq i} (f_i(\thetav) - f_j(\thetav))\bigg|^2  \nonumber \\
& \leq \sum_i f_i(\thetav) + \frac{\beta}{n^2} \sum_{i}\bigg(\sum_{j \neq i} \bigg| f_i(\thetav) - f_j(\thetav)\bigg|\bigg)^2 \nonumber\\ 
& \leq \sum_i f_i(\thetav) + \frac{\beta}{n^2}\sum_i (n-1)\sum_{j \neq i} \bigg| f_i(\thetav) - f_j(\thetav)\bigg|^2 \nonumber\\
& \leq \sum_i f_i(\thetav) + \frac{\beta}{n}\sum_i\sum_{j \neq i} \bigg| f_i(\thetav) - f_j(\thetav)\bigg|^2 \nonumber\\ 
& = \sum_i f_i(\thetav) + \frac{2\beta}{n} \sum_{j \neq i} \bigg| f_i(\thetav) - f_j(\thetav)\bigg|^2 \nonumber \\
& \leq \sum_i f_i(\thetav) + \frac{2\beta}{n} \sum_{j \neq i} \bigg| f_i(\thetav) - f_j(\thetav)\bigg|,
\end{align}
where the last inequality is always true if $\forall i,j: |f_i(\thetav)-f_j(\thetav)|<1$, which always happens if $\forall i: f_i(\thetav)<1$.

\gradvred*
\begin{proof}
From \eqref{eq:v-red} and with $p_i=\frac{1}{n}$, we have:

\begin{align}
    n\nabla F(\thetav) &= \sum_i \nabla f_i(\thetav)+2\beta\sum_i\Big[ \Big(f_i(\thetav)-\overline{f}(\thetav)\Big)\Big(\nabla f_i(\thetav)-\nabla\overline{f}(\thetav)\Big)\Big]
    \nonumber \\
    & = \sum_i \nabla f_i(\thetav) + 2\beta \sum_i \Big[ \Big(f_i(\thetav)-\overline{f}(\thetav)\Big) \nabla f_i(\thetav) - \Big(f_i(\thetav)-\overline{f}(\thetav)\Big) \nabla \overline{f}(\thetav) \Big]
    \nonumber \\
    & = \sum_i \Big( 1+2\beta(f_i(\thetav)-\overline{f}(\thetav))\Big)\nabla f_i(\thetav) - 2\beta \sum_i \Big(f_i(\thetav) - \overline{f}(\thetav) \Big)\nabla \overline{f}(\thetav)
    \nonumber \\
    & = \sum_i \Big( 1+2\beta(f_i(\thetav)-\overline{f}(\thetav))\Big)\nabla f_i(\thetav)
\end{align}

Hence,

\begin{align}
    \nabla F(\thetav) = \sum_i \frac{1+2\beta(f_i(\thetav)-\overline{f}(\thetav))}{n}\nabla f_i(\thetav)
\end{align}
\end{proof}

\gradsemivred*
\begin{proof}
From \eqref{eq:semivar-red} and with $p_i=\frac{1}{n}$, we have:

\begin{align}
    n\nabla F(\thetav) &= \sum_i \nabla f_i(\thetav)+2\beta\sum_{i\in>_{\mathcal{C}}}\Big[ \Big(f_i(\thetav)-\overline{f}(\thetav)\Big)\Big(\nabla f_i(\thetav)-\nabla\overline{f}(\thetav)\Big)\Big]
    \nonumber \\
    & = \sum_i \nabla f_i(\thetav) + 2\beta \sum_{i\in>_{\mathcal{C}}} \Big[ \Big(f_i(\thetav)-\overline{f}(\thetav)\Big) \nabla f_i(\thetav) - \Big(f_i(\thetav)-\overline{f}(\thetav)\Big) \nabla \overline{f}(\thetav) \Big]
    \nonumber \\
    & = \sum_{i\notin>_{\mathcal{C}}}\nabla f_i(\thetav) + \sum_{i\in>_{\mathcal{C}}} \Big( 1+2\beta(f_i(\thetav)-\overline{f}(\thetav))\Big)\nabla f_i(\thetav) - 2\beta \sum_{i\in>_{\mathcal{C}}} \Big(f_i(\thetav) - \overline{f}(\thetav) \Big)\nabla \overline{f}(\thetav)
    \nonumber \\
\end{align}

The last term in the above equation can be written as:

\begin{align}
     & - 2\beta \sum_{i\in>_{\mathcal{C}}} \Big(f_i(\thetav) - \overline{f}(\thetav) \Big)\nabla \overline{f}(\thetav) = - \Big [\frac{2\beta}{n} \Big( \sum_{i\in>_{\mathcal{C}}} f_i(\thetav)-\overline{f}(\thetav) \Big) \times \big(\sum_j \nabla f_j(\thetav)\Big)\Big]
\end{align}
Hence,

\begin{align}
    n\nabla F(\thetav) & = \sum_{i\in>_{\mathcal{C}}} \Big( 1+2\beta(f_i(\thetav)-\overline{f}(\thetav))-\frac{2\beta}{n} \Big( \sum_{j\in>_{\mathcal{C}}} f_j(\thetav)-\overline{f}(\thetav)\Big)\Big)\nabla f_i(\thetav) \\
    \nonumber
    & + \sum_{i\notin>_{\mathcal{C}}} \Big(1-\frac{2\beta}{n} \Big( \sum_{j\in>_{\mathcal{C}}} f_j(\thetav)-\overline{f}(\thetav)\Big)\Big)\nabla f_i(\thetav)
\end{align}

Therefore, 

\begin{align}
    \nabla F(\thetav) & = \sum_{i\in>_{\mathcal{C}}} \Big(\frac{1+2\beta(f_i(\thetav)-\overline{f}(\thetav))-\frac{2\beta}{n} \Big( \sum_{j\in>_{\mathcal{C}}} f_j(\thetav)-\overline{f}(\thetav)\Big)}{n}\Big)\nabla f_i(\thetav) \\
    \nonumber
    & + \sum_{i\notin>_{\mathcal{C}}} \Big(\frac{1-\frac{2\beta}{n} \Big( \sum_{j\in>_{\mathcal{C}}} f_j(\thetav)-\overline{f}(\thetav)\Big)}{n}\Big)\nabla f_i(\thetav)
\end{align}

\end{proof}

\newpage
\section{Experimental setup}
\label{sec:appendix_setup}
In this section, we provide more experimental details that are deferred from the main paper. For each experiment, we report the average result obtained from three runs with different random seeds. For our experiments, we used an internal GPU server with six NVIDIA Tesla P100. The experiments last about 4 weeks in total.

\subsection{Datasets and models}
In this subsection, we describe the datasets we use in our experiments. For all the datasets we use a batch size of 64.

\textbf{CIFAR-10/100} \cite{krizhevsky2009learning} are two image classification datasets vastly used in the literature as benchmark datasets. Each of these datasests contains 50000 sample images with 10/100 balanced classes for CIFAR-10 and CIFAR-100, respectively. We use Dirichlet allocation \cite{wang_federated_2019} to distribute the data among 50 clients with label shift: we split the set of samples from class $k$ ($\Sc_k$) to $n$ subsets ($\Sc_k=\Sc_{k,1}\cup\Sc_{k,2}\cup\ldots\cup\Sc_{k,n}$), where $n$ is the number of clients and $\Sc_{k,j}$ corresponds to the client $j$. We do the split based on Dirichlet distribution with parameter 0.05 (\texttt{Dir(0.05)}). When the split is done for all classes, we gather the samples corresponding to each client from different classes: assuming there are $C$ classes in total $\Sc_{1,j}\cup\Sc_{2,j}\cup\ldots\cup\Sc_{C,j}$ is the data allocated to the client $j$. Having allocated the data of each client, we split them into train and test set for each client. The train-test split ratio is $50$-$50$ and $60$-$40$ for CIFAR-10 and CIFAR-100, respectively. 

\textbf{CINIC-10} \cite{darlow2018cinic}
is another benchmark vision dataset that we use in our experiments. There are a total of 270,000 sample images, which we distribute with label shift between 50 clients based on \texttt{Dir(0.5)} distribution \cite{wang_federated_2019}. We then randomly split the data of each client into train and test sets with split ratio 50-50.

\textbf{StackOverflow} \cite{authors2019tensorflow}
is a language dataset consisting of Shakespeare dialogues for the task of next word prediction. There is a natural heterogeneous partition of the dataset and we treat each speaking role as a client. We filter out the clients (speaking roles) with less than 10,000 samples from the original dataset and randomly select 20 clients from the remaining. Finally, we split the data of each client into train and test sets with a ratio of 50-50.

\Cref{tab:datasets} provides a summary of the datasets we used and the models used for each of them.


\begin{table*}[h]
\centering
\caption{Details of the experiments and the datasets. ResNet-18: residual neural network 
; GN: Group Normalization
; RNN: Recurrent Neural Network
; LSTM: Long Short-Term Memory layer
; FC: fully connected layer.}
\label{tab:datasets}
\small
\setlength\tabcolsep{2pt}
\begin{tabular}{cccccc}
\toprule
\bf{Dataset} & \bf{Train samples} & \bf{Test samples} & \bf{Partition method} & \bf{clients} & \bf{Model} 
\\ 
\midrule
CIFAR-10 & 24959 & 25041 & \texttt{Dir(0.05)} & 50 & ResNet-18 + GN\\

CIFAR-100 & 39445 & 10555 & \texttt{Dir(0.05)} & 50 & ResNet-18 + GN
\\

CINIC-10 & 134713 & 134966 & \texttt{Dir(0.5)} &  50 & ResNet-18 + GN
\\


StackOverflow & 109671 & 109621 & realistic partition & 20 & RNN (1 LSTM + 2 FC)
\\
\bottomrule
\end{tabular}
\end{table*}

\subsection{Algorithms and their hyperparameters}

We use most recent fair \FL algorithms existing in the literature as our baseline algorithms, including: \FedAvg \cite{McMahanMRHA17}, \qFFL \cite{li2019fair}, \AFL \cite{MohriSS19}, \PropFair \cite{Zhang2022ProportionalFI}, \TERM\cite{li2020tilted}, \GiFair\cite{Yue2021GIFAIRFLAA}. For each pair of dataset and algorithm, we find the best learning rate based on a grid search. In the following, we have reported the learning rate grid we use for each dataset:

\begin{itemize}
\item CIFAR-10: \texttt{\{1e-3, 2e-3, 5e-3, 1e-2, 2e-2, 5e-2\}};
\item CIFAR-100: \texttt{\{1e-3, 2e-3, 5e-3, 1e-2, 2e-2, 5e-2\}};
\item CINIC-10:  \texttt{\{1e-3, 2e-3, 5e-3, 1e-2, 2e-2, 5e-2\}};
\item StackOverflow: \texttt{\{1e-2, 5e-2, 1e-1, 5e-1, 1\}}.
\end{itemize}

The best learning rate used for each (dataset, algorithm) pair is reported in \Cref{tab:best_lr}.



\begin{table*}[h]
\centering
\caption{The best learning rates used for training each algorithm on different datasets.}
\label{tab:best_lr}
\small
\setlength\tabcolsep{2pt}
\begin{tabular}{cccccccccc}
\toprule
\bf{Datasets} & \bf{\FedAvg} & \bf{\qFFL} & \bf{\AFL} & \bf{\TERM} & \bf{\PropFair} & \bf{\GiFair} & \bf{$\Delta$-\FL} & \bf{\VRed} & \bf{\SemiVRed} 
\\ 
\midrule
CIFAR-10 & {\tt 5e-3} & {\tt 5e-3} & {\tt 5e-3} & {\tt 1e-2} & {\tt 1e-2} & {\tt 5e-3} & {\tt 5e-3} & {\tt 5e-3} & {\tt 5e-3}
\\

CIFAR-100 & {\tt 2e-3} & {\tt 2e-3} & {\tt 5e-3} & 
{\tt 1e-2} & {\tt 1e-2} & {\tt 5e-3} & {\tt 2e-3} & {\tt 5e-3} & {\tt 5e-3}
\\

CINIC-10 & {\tt 1e-2} & {\tt 5e-3} & {\tt 1e-2} & {\tt 1e-2}  & {\tt 2e-2} & {\tt 2e-2} & {\tt 1e-2} & {\tt 5e-3} & {\tt 5e-3} 
\\


StackOverflow & {\tt 2e-1} & {\tt 5e-2} & {\tt 5e-2} & {\tt 2e-1}  & {\tt 5e-1} & {\tt 2e-1} & {\tt 2e-1} & {\tt 5e-1} & {\tt 5e-1} 
\\
\bottomrule
\end{tabular}
\end{table*}

We now explain the algorithms we use and how we tune their hyperparameters. We adapt \TERM with only client-level fairness ($\alpha>0$) and no sample-level fairness ($\tau=0$). We tune the hyperparameter $\alpha$ for each dataset based on a grid search in the grid $\{ 0.01, 0.1, 0.5, 1\}$. We have reported the best value of $\alpha$ for each dataset in \Cref{tab:hyperparameters}. For \AFL, there are two hyperparameters: $\gamma_w$ and $\gamma_\lambda$. We tune the learning rate $\gamma_w$ from the corresponding grid and choose the default value $\gamma_\lambda=0.1$. For \qFFL, we use the $q$-FedAvg algorithm \cite{li2019fair}. We also tune the hyperparameter $q$ from the grid $\{0.01, 0.1, 1\}$. We find that for all the used datasets, $q=0.1$ has the best peformance (as reported in \Cref{tab:hyperparameters}). We also tried larger values out of the grid and they often lead to divergence of the \qFFL algorithm. We adopt the Global \GiFair model \cite{Yue2021GIFAIRFLAA}, which results in a single global model. In order to have client-level fairness, we treat each client as a group of size 1. For tuning the regularization weight of \GiFair ($\lambda$), we follow \cite{Yue2021GIFAIRFLAA}. As stated in the paper, there is an upper-bound for $\lambda$ (see \S 3 in the paper). For our experiments, the upper-bound is $\lambda\leq \min_i \{\frac{w_i}{n-1}\}$, where $w_i$ is the ratio of total samples allocated to the client $i$ and $n$ is the number of clients. We try four different values in the interval and choose the best one. When the number of clients is large, this upper-bound is small, and for all of our datasets, this upper-bound was the best value, as reported in \Cref{tab:hyperparameters}. We tune $M$ for the \PropFair algorithm based on a grid search in $\{2,3,4,5\}$. For $\Delta$-\FL, we tried different values of $\alpha$
 in the grid $\{0.1, 0.2, 0.3, 0.4, 0.5, 0.6, 0.7, 0.8, 0.9\}$ and chose the best one. Therefore, we have reported \textbf{the best results} that one could get from the algorithm. Finally, for our \VRed and \SemiVRed algorithms, we tune the regularization weight $\beta$ based on grid search on the grid $\{0.01, 0.05, 0.1, 0.2, 0.5, 1\}$. Larger values of $\beta$ often resulted in the divergence of the algorithms. We have reported the best value of all of the hyperparameters for each dataset in \Cref{tab:hyperparameters}.

\begin{table*}[ht]
\centering
\caption{The best values of hyperparameters used for different datasets, chosen based on grid search.}
\label{tab:hyperparameters}
\small
\setlength\tabcolsep{2pt}
\begin{tabular}{cccccc}
\toprule
\bf{Algorithm} & \bf{CIFAR-10} & \bf{CIFAR-100} & \bf{CINIC-10} & \bf{StackOverflow}
\\ 
\midrule
\textbf{\qFFL} $q$ & \tt 1e-1 & \tt 1e-1 & \tt 1e-1  & \tt 1e-1 \\
\midrule
\textbf{\TERM} $\alpha$ & \tt 1e-2 & \tt 5e-1 & \tt 5e-1 & \tt 5e-1
\\
\midrule
\textbf{\GiFair} $\lambda$ & \tt 6e-5 & \tt 2.6e-4 & \tt 5e-5 & \tt 2.4e-3
\\
\midrule
\textbf{\PropFair} $M$ & \tt 3 & \tt 3 & \tt 5 & \tt 4
\\
\midrule
\textbf{$\Delta$-\FL} $\alpha$ & \tt 2e-1 & \tt 4e-1 & \tt 5e-1 & \tt 6e-1
\\
\midrule
\textbf{\VRed} $\beta$ & \tt 5e-1 & \tt 1e-1 & \tt 2e-1 & \tt 1e-1
\\
\midrule
\textbf{\SemiVRed} $\beta$ & \tt 5e-1 & \tt 1e-2 & \tt 2e-1 & \tt 2e-1
\\

\bottomrule
\end{tabular}
\end{table*}

\subsection{Detailed results}
In \Cref{tab:detailed_results}, we report detailed results obtained from the algorithms we study in this work. We use a default batch size of 64 for all the experiments. The statistics we report include: 1. the average test accuracy across clients (overall average performance) 2. the standard deviation of test accuracies across clients 3. the lowest (worst) test accuracy among clients 4. the lowest 10\% test accuracies 5. the lowest 20\% test accuracies 6. the highest 10\% test accuracies. For each experiment, we report the average result obtained from three runs with different random seeds. As can be observed, our proposed algorithms \VRed and \SemiVRed consistently beat almost all the baseline algorithms across different datasets in terms of various fairness metrics. Also, \SemiVRed can improve the overall average performance (mean test accuracy) for three of the datsets as well.

Following \Cref{fig:comparison}, we have compared our proposed algorithms with the baseline algorithms in terms of their worst 20\% test accuracies as well and the visualized results are shown in \Cref{fig:worst20accs}.

\begin{table*}[hbt!]
\centering
\caption{Comparison among federated learning algorithms on CIFAR-10, CIFAR-100, CINIC-10 and StackOverflow datasets with test accuracies (\%) from clients. All algorithms are fine-tuned. {\bf Mean}: the average test accuracy across all clients; {\bf Std}: standard deviation of clients test accuracies; {\bf Worst}: the worst test accuracy among clients; {\bf Worst (10/20\%)}: the worst 10/20\% test accuracies of clients; {\bf Best (10\%)}: the best 10\% test accuracies of clients. 
}
\label{tab:detailed_results}
\small
\setlength\tabcolsep{2pt}
\begin{tabular}{cccccccc}
\toprule
\bf Dataset & \bf Algorithm & \bf Mean & \bf Std & \bf Worst & \bf Worst (10\%) & \bf Worst (20\%) & \bf Best (10\%)
\\ 
\midrule\midrule
 & \FedAvg & 
 43.45$_{\pm 0.60}$ & 
 14.33$_{\pm 0.62}$ &
 9.35$_{\pm 3.13}$ & 
 18.86$_{\pm 0.99}$ &
 23.77$_{\pm 0.70}$ &
 68.97$_{\pm 0.81}$\\

\multirow{5}{*}{\rotatebox[origin=c]{90}{CIFAR-10}} & \qFFL & 
45.46$_{\pm 0.74}$ & 
14.31$_{\pm 2.03}$ &
18.71$_{\pm 3.36}$ & 
21.23$_{\pm 3.06}$ &
25.95$_{\pm 3.51}$ &
72.31$_{\pm 2.88}$\\

& \AFL & 
- & 
- &
- & 
- &
- &
-
\\
& \GiFair&  
45.05$_{\pm 0.64}$ & 
12.93$_{\pm 0.44}$ &
16.79$_{\pm 3.55}$ &
22.65$_{\pm 2.03}$ &
26.52$_{\pm 0.76}$ &
65.62$_{\pm 2.59}$
\\
& \TERM & 
\bf45.61$_{\pm 1.03}$ & 
\bf12.24$_{\pm 0.56}$ &
13.80$_{\pm 5.25}$ &
24.89$_{\pm 1.37}$ &
29.34$_{\pm 0.61}$ &
68.65$_{\pm 1.27}$
\\

& \PropFair &  
36.95$_{\pm 0.21}$ & 
15.16$_{\pm 1.33}$ & 
1.14$_{\pm 1.62}$ & 
12.49$_{\pm 0.28}$ &
16.66$_{\pm 1.31}$ &
66.04$_{\pm 4.24}$
\\

\color{blue}
& $\Delta$-\FL &  
40.32$_{\pm 0.31}$ & 
16.51$_{\pm 1.26}$ & 
8.8$_{\pm 0.37}$ & 
16.94$_{\pm 0.55}$ &
21.31$_{\pm 0.31}$ &
72.35$_{\pm 2.11}$

\color{black}
\\
\cmidrule{2-8}
& \VRed &  
44.43$_{\pm 0.88}$ &
13.05$_{\pm 1.32}$ & 
18.61 $_{\pm 3.12}$ &
24.28 $_{\pm 2.22}$ & 
27.46$_{\pm 1.56}$ & 
69.31$_{\pm 3.48}$
\\

& \SemiVRed &  
45.47$_{\pm 0.10}$ &
12.58$_{\pm 0.23}$ & 
\bf19.04$_{\pm 6.73}$ &
\bf27.08$_{\pm 1.76}$ &
\bf30.34$_{\pm 1.05}$ &
\bf72.50$_{\pm 0.88}$
\\

\midrule
 & \FedAvg & 
 20.20$_{\pm 0.31}$ & 
 6.50$_{\pm 0.21}$ &
 \bf10.36$_{\pm 0.69}$ & 
 11.07$_{\pm 0.54}$ &
 12.49$_{\pm 0.51}$ &
 33.88$_{\pm 0.09}$\\

\multirow{5}{*}{\rotatebox[origin=c]{90}{CIFAR-100}} & \qFFL & 
20.25$_{\pm 0.11}$ & 
6.30$_{\pm 0.27}$ &
9.66$_{\pm 0.33}$ &
11.09$_{\pm 0.67}$ &
12.52$_{\pm 0.46}$ &
33.96$_{\pm 0.90}$ \\ 

& \AFL & 
18.98$_{\pm 0.71}$ & 
\bf4.91$_{\pm 0.37}$ &
9.81$_{\pm 0.69}$ & 
11.31$_{\pm 0.18}$ &
12.72$_{\pm 0.21}$ &
28.68$_{\pm 1.71}$
\\
& \GiFair&  
19.81$_{\pm0.32}$ & 
5.74$_{\pm0.16}$ & 
9.35$_{\pm0.34}$ &
11.19$_{\pm0.24}$ &
12.59$_{\pm 0.49}$ &
32.30$_{\pm0.32}$
\\

& \TERM & 
18.00$_{\pm0.41}$ & 
6.05$_{\pm0.18}$ &
8.86$_{\pm0.50}$ &
10.02$_{\pm0.44}$ & 
11.04$_{\pm 0.51}$ &
31.58$_{\pm0.98}$
\\

& \PropFair 
& 14.97$_{\pm 0.68}$& 
6.44$_{\pm 0.34}$& 
5.40$_{\pm 1.28}$& 
7.00$_{\pm 1.11}$& 
8.06$_{\pm 1.07}$& 
28.89$_{\pm 0.91}$
\\

& $\Delta$-\FL &  
18.39$_{\pm 0.15}$ & 
5.42$_{\pm 0.65}$ & 
8.11$_{\pm 0.84}$ & 
10.06$_{\pm 0.78}$ &
11.28$_{\pm 0.81}$ &
28.95$_{\pm 1.27}$
\\
\cmidrule{2-8} 
& \VRed &  
20.42$_{\pm0.36}$ &
6.08$_{\pm0.05}$ & 
9.43$_{\pm1.01}$ &
11.21$_{\pm0.74}$ &
12.81$_{\pm 0.85}$ &
33.59$_{\pm1.11}$
\\

& \SemiVRed &   
\bf20.85$_{\pm0.39}$ &
6.26$_{\pm0.18}$ & 
9.12$_{\pm1.47}$ &
\bf11.86$_{\pm0.74}$ & 
\bf13.46$_{\pm 0.63}$ &
\bf34.57$_{\pm1.20}$
\\

\midrule
 & \FedAvg
 & 86.26$_{\pm 0.03}$ 
 & 15.20$_{\pm 0.07}$ 
 & 50.48$_{\pm 0.29}$ 
 & 56.87$_{\pm0.36}$ 
 & 62.78$_{\pm 0.16}$
 & 100.0$_{\pm0.00}$\\
 
\multirow{5}{*}{\rotatebox[origin=c]{90}{CINIC-10}} 
& \qFFL 
& \bf86.63$_{\pm0.06}$ 
& \bf14.88$_{\pm0.08}$ 
& 51.57$_{\pm0.82}$ 
& 57.77$_{\pm0.36}$
& \bf 63.62$_{\pm 0.18}$
& \bf   100.0$_{\pm0.01}$\\

& \AFL 
& 86.45$_{\pm0.12}$ 
& 15.10$_{\pm0.11}$ 
& 51.57$_{\pm0.45}$ 
& 57.58$_{\pm0.29}$
& 63.04$_{\pm 0.28}$
& 100.0$_{\pm0.00}$
\\

& \GiFair&  
86.28$_{\pm0.11}$ & 
15.20$_{\pm0.13}$ & 
49.66$_{\pm1.21}$ &
56.97$_{\pm0.29}$ &
62.74$_{\pm 0.36}$ &
100.0$_{\pm0.00}$
\\

& \TERM & 
86.34$_{\pm0.04}$ & 
15.12$_{\pm0.01}$ &
49.90$_{\pm0.42}$ &  
57.21$_{\pm0.11}$ &
62.98$_{\pm 0.04}$ &
100.0$_{\pm0.00}$
\\
& \PropFair & 
86.01$_{\pm0.17}$ & 
15.34$_{\pm0.19}$ &
49.97$_{\pm1.23}$ &
56.53$_{\pm0.65}$ &
62.27$_{\pm 0.55}$ &
99.99$_{\pm0.01}$
\\

& $\Delta$-\FL &  
86.11$_{\pm 0.38}$ & 
15.11$_{\pm 0.22}$ & 
50.12$_{\pm 0.62}$ & 
57.10$_{\pm 0.76}$ &
62.45$_{\pm 0.85}$ &
100.0$_{\pm 0.01}$
\\

\cmidrule{2-8} 
& \VRed & 
85.79$_{\pm0.35}$ & 
15.02$_{\pm0.06}$ &
51.57$_{\pm0.50}$ &
57.66$_{\pm0.30}$ &
62.75$_{\pm 0.36}$ &
99.98$_{\pm0.01}$\\

& \SemiVRed & 
85.83$_{\pm0.33}$ & 
14.95$_{\pm0.07}$ &
\bf51.59$_{\pm0.98}$ &
\bf58.00$_{\pm0.21}$ &
62.70$_{\pm 0.14}$ &
99.96$_{\pm0.01}$ 
\\

\midrule
 & \FedAvg& 
 40.34$_{\pm0.06}$ & 
 6.98$_{\pm0.03}$ &
 25.64$_{\pm0.11}$ & 
 27.12$_{\pm0.06}$ & 
 30.35$_{\pm 0.03}$ &
 49.70$_{\pm0.07}$\\

\multirow{5}{*}{\rotatebox[origin=c]{90}{StackOverflow}}
& \qFFL & 
37.79$_{\pm0.80}$ & 
7.38$_{\pm0.09}$ &
22.54$_{\pm1.03}$ & 
24.12$_{\pm1.00}$ &
27.14$_{\pm 0.92}$ &
47.06$_{\pm0.66}$  
\\
& \AFL & 
- &
- &
- &
- &
- &
-
\\
& \TERM & 
40.34$_{\pm0.05}$ & 
6.96$_{\pm0.06}$ &
25.56$_{\pm0.21}$ & 
27.12$_{\pm0.20}$ &
30.41$_{\pm 0.12}$ &
49.76$_{\pm0.10}$
\\
& \GiFair&  
40.34$_{\pm0.04}$ & 
6.97$_{\pm0.03}$ & 
25.71$_{\pm0.13}$ &
27.10$_{\pm0.11}$ &
30.34$_{\pm 0.08}$ &
49.71$_{\pm0.09}$
\\

& \PropFair & 
41.76$_{\pm0.01}$ &
6.80$_{\pm0.05}$ &
27.30$_{\pm0.21}$ &
28.75$_{\pm0.19}$ &
32.14$_{\pm 0.10}$ &
50.76$_{\pm0.08}$
\\

& $\Delta$-\FL &  
39.94$_{\pm 0.11}$ & 
6.88$_{\pm 0.05}$ & 
25.36$_{\pm 0.08}$ & 
26.94$_{\pm 0.07}$ &
29.95$_{\pm 0.06}$ &
49.08$_{\pm 0.04}$
\\

\cmidrule{2-8} 
& \VRed &  
42.90$_{\pm0.05}$ &
6.64$_{\pm0.01}$ & 
29.08$_{\pm0.09}$ &
\bf30.39$_{\pm0.05}$ &
33.55$_{\pm 0.05}$ &
51.66$_{\pm0.03}$
\\
& \SemiVRed &  
\bf42.90$_{\pm0.03}$ &
\bf6.60$_{\pm0.01}$ & 
\bf29.10$_{\pm0.06}$ &
30.34$_{\pm0.09}$ & 
\bf 35.55$_{\pm 0.05}$ &
\bf51.70$_{\pm0.04}$
\\
\bottomrule
\end{tabular}
\end{table*}

\pgfplotstableread[row sep=\\, col sep=&]
{alg & worst10 \\
\rotatebox[origin=c]{90}{FedAvg} & 23.77 \\
\rotatebox[origin=c]{90}{$q$-FFL} & 25.95 \\
\rotatebox[origin=c]{90}{AFL} & 0 \\
\rotatebox[origin=c]{90}{GiFair} & 26.52 \\
\rotatebox[origin=c]{90}{TERM} & 29.34 \\
\rotatebox[origin=c]{90}{PropFair} & 16.66 \\
\rotatebox[origin=c]{90}{$\Delta$-FL} & 21.31\\
\rotatebox[origin=c]{90}{VRed} & 27.46 \\
\rotatebox[origin=c]{90}{Semi-VRed} & 30.34 \\
}\cifartenworsttendata

\pgfplotstableread[row sep=\\, col sep=&]
{alg & worst10 \\
\rotatebox[origin=c]{90}{FedAvg} & 12.49 \\
\rotatebox[origin=c]{90}{$q$-FFL} & 12.52 \\
\rotatebox[origin=c]{90}{AFL} & 12.72 \\
\rotatebox[origin=c]{90}{GiFair} & 12.59 \\
\rotatebox[origin=c]{90}{TERM} & 11.04 \\
\rotatebox[origin=c]{90}{PropFair} & 8.06 \\
\rotatebox[origin=c]{90}{$\Delta$-FL} & 11.28\\
\rotatebox[origin=c]{90}{VRed} & 12.81 \\
\rotatebox[origin=c]{90}{Semi-VRed} & 13.46 \\
}\cifarhundredworsttendata

\pgfplotstableread[row sep=\\, col sep=&]
{alg & worst10 \\
\rotatebox[origin=c]{90}{FedAvg} & 62.78 \\
\rotatebox[origin=c]{90}{$q$-FFL} & 63.62 \\
\rotatebox[origin=c]{90}{AFL} & 63.04 \\
\rotatebox[origin=c]{90}{GiFair} & 62.74 \\
\rotatebox[origin=c]{90}{TERM} & 62.98 \\
\rotatebox[origin=c]{90}{PropFair} & 62.27 \\
\rotatebox[origin=c]{90}{$\Delta$-FL} & 62.45\\
\rotatebox[origin=c]{90}{VRed} & 62.75 \\
\rotatebox[origin=c]{90}{Semi-VRed} & 62.70 \\
}\cinictenworsttendata

\pgfplotstableread[row sep=\\, col sep=&]
{alg & worst10 \\
\rotatebox[origin=c]{90}{FedAvg} & 30.35 \\
\rotatebox[origin=c]{90}{$q$-FFL} & 27.14 \\
\rotatebox[origin=c]{90}{AFL} & 0 \\
\rotatebox[origin=c]{90}{GiFair} & 30.41 \\
\rotatebox[origin=c]{90}{TERM} & 30.34 \\
\rotatebox[origin=c]{90}{PropFair} & 32.14 \\
\rotatebox[origin=c]{90}{$\Delta$-FL} & 29.95\\
\rotatebox[origin=c]{90}{VRed} & 33.55 \\
\rotatebox[origin=c]{90}{Semi-VRed} & 35.55 \\
}\stackoverflowworsttendata

\begin{figure*}[h]

\subfigure{
\begin{tikzpicture}
\begin{axis}
[title= CIFAR-10,
width  = 0.49\textwidth,
height = 5cm,
draw=red,
bar width=7pt,
ymajorgrids = true,
ylabel={test accuracy},
legend style={
at={(0.985,0.88)},
anchor=south east,
column sep=1ex
},
ymin=10, 
ymax=35,
ybar,
symbolic x coords = {\rotatebox[origin=c]{90}{FedAvg}, \rotatebox[origin=c]{90}{$q$-FFL}, \rotatebox[origin=c]{90}{AFL}, \rotatebox[origin=c]{90}{GiFair}, \rotatebox[origin=c]{90}{TERM}, \rotatebox[origin=c]{90}{PropFair}, \rotatebox[origin=c]{90}{$\Delta$-FL}, \rotatebox[origin=c]{90}{VRed}, \rotatebox[origin=c]{90}{Semi-VRed}},
xtick=data,
]


\addplot [red!10!black,fill=red!30!white] table[x=alg, y=worst10]{\cifartenworsttendata};
\addplot[red,sharp plot,dashed]
coordinates {(\rotatebox[origin=c]{90}{FedAvg},30.34) (\rotatebox[origin=c]{90}{Semi-VRed},30.34)};
\legend{worst 20\% test accuracy}
\end{axis}

\end{tikzpicture}
}
\hfill
\subfigure{
\begin{tikzpicture}
\begin{axis}
[title= CIFAR-100,
width  = 0.49\textwidth,
height = 5cm,
bar width=7pt,
ymajorgrids = true,
legend style={
at={(0.985,0.95)},
anchor=south east,
column sep=1ex
},
ybar,
symbolic x coords = {\rotatebox[origin=c]{90}{FedAvg}, \rotatebox[origin=c]{90}{$q$-FFL}, \rotatebox[origin=c]{90}{AFL}, \rotatebox[origin=c]{90}{GiFair}, \rotatebox[origin=c]{90}{TERM}, \rotatebox[origin=c]{90}{PropFair}, \rotatebox[origin=c]{90}{$\Delta$-FL}, \rotatebox[origin=c]{90}{VRed}, \rotatebox[origin=c]{90}{Semi-VRed}},
xtick=data,
]

\addplot [red!10!black,fill=red!30!white]table[red, x=alg, y=worst10]{\cifarhundredworsttendata};
\addplot[red,sharp plot,dashed]
coordinates {(\rotatebox[origin=c]{90}{FedAvg},13.46) (\rotatebox[origin=c]{90}{Semi-VRed},13.46)};
\end{axis}
\end{tikzpicture}
}

\subfigure{
\begin{tikzpicture}
\begin{axis}
[title= CINIC-10,
width  = 0.49\textwidth,
height = 5cm,
bar width=7pt,
ymajorgrids = true,
ylabel={test accuracy},
ymin=55, 
ybar,
symbolic x coords = {\rotatebox[origin=c]{90}{FedAvg}, \rotatebox[origin=c]{90}{$q$-FFL}, \rotatebox[origin=c]{90}{AFL}, \rotatebox[origin=c]{90}{GiFair}, \rotatebox[origin=c]{90}{TERM}, \rotatebox[origin=c]{90}{PropFair}, \rotatebox[origin=c]{90}{$\Delta$-FL}, \rotatebox[origin=c]{90}{VRed}, \rotatebox[origin=c]{90}{Semi-VRed}},
xtick=data,
]

\addplot [red!10!black,fill=red!30!white]table[x=alg,y=worst10]{\cinictenworsttendata};
\addplot[red,sharp plot,dashed]
coordinates {(\rotatebox[origin=c]{90}{FedAvg},63.62) (\rotatebox[origin=c]{90}{Semi-VRed},63.62)};
\end{axis}
\end{tikzpicture}
}
\hfill
\subfigure{
\begin{tikzpicture}
\begin{axis}
[title= StackOverflow,
width  = 0.49\textwidth,
height = 5cm,
bar width=7pt,
ymajorgrids = true,
legend style={
at={(0.985,0.78)},
anchor=south east,
column sep=1ex
},
ymin=20, 
ybar,
symbolic x coords = {\rotatebox[origin=c]{90}{FedAvg}, \rotatebox[origin=c]{90}{$q$-FFL}, \rotatebox[origin=c]{90}{AFL}, \rotatebox[origin=c]{90}{GiFair}, \rotatebox[origin=c]{90}{TERM}, \rotatebox[origin=c]{90}{PropFair}, \rotatebox[origin=c]{90}{$\Delta$-FL}, \rotatebox[origin=c]{90}{VRed}, \rotatebox[origin=c]{90}{Semi-VRed}},
xtick=data,
]

\addplot [red!10!black,fill=red!30!white]table[x=alg, y=worst10]{\stackoverflowworsttendata};
\addplot[red,sharp plot,dashed]
coordinates {(\rotatebox[origin=c]{90}{FedAvg},35.55) (\rotatebox[origin=c]{90}{Semi-VRed},35.55)};
\end{axis}
\end{tikzpicture}
}

\caption{Worst 20\% test accuracies for different algorithms. \textbf{top left:} CIFAR-10, \textbf{top right:} CIFAR-100, \textbf{bottom left:} CINIC-10, \textbf{bottom right:} StackOverflow. Due to divergence, results for \AFL on CIFAR-10 and StackOverFlow are not shown. All subfigures share the same legends and axis labels.}
\label{fig:worst20accs}
\end{figure*}
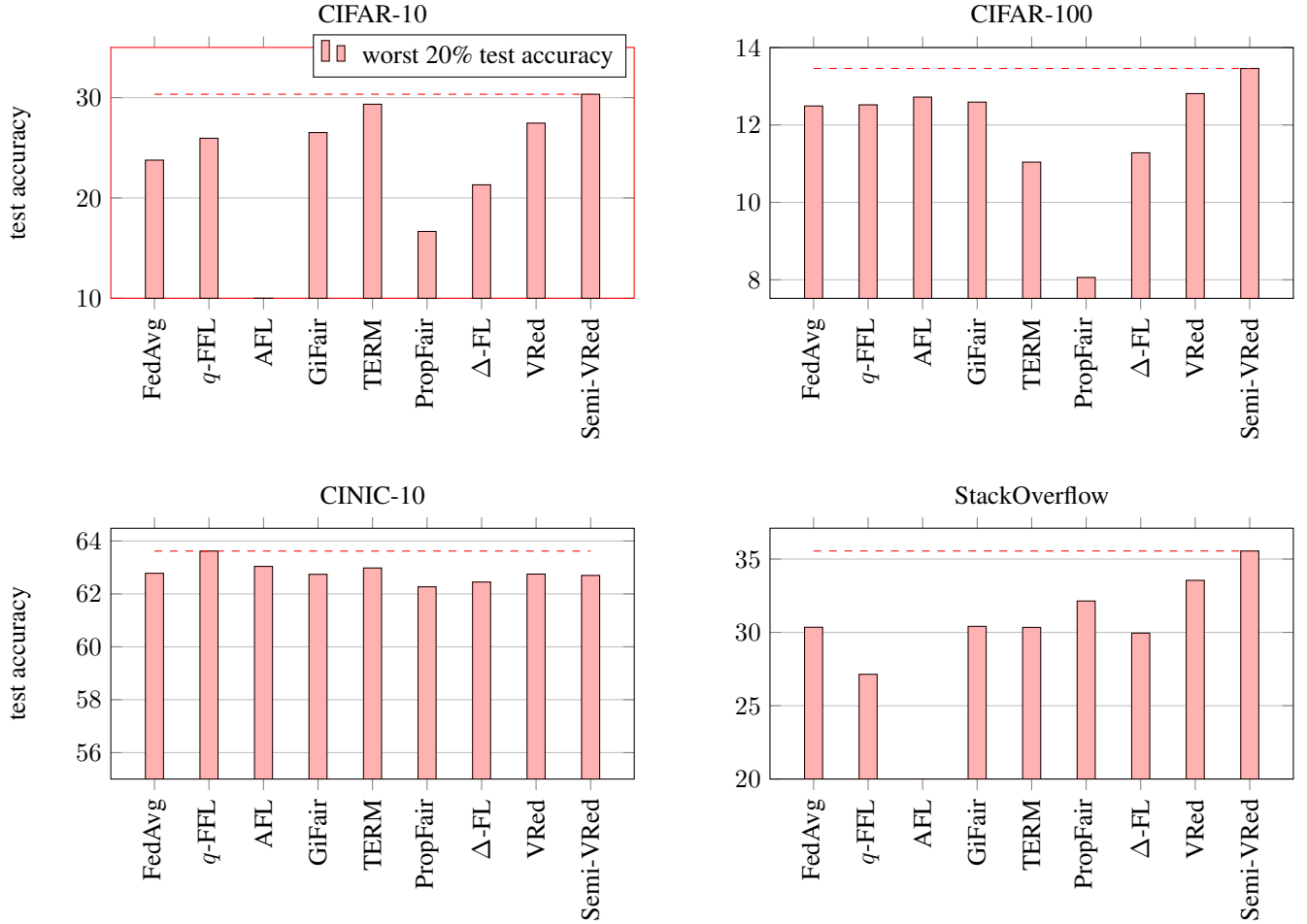

\subsection{Relation between \VRed and Robust Optimization}\label{vr_and_dro}
Empirical optimization is usually used as a data-driven approach for tuning models for decision making, where an expected loss is minimized based on some available train data. The trained model is then used for prediction tasks on some test data. However, if the empirical distribution of the train data is substantially different from that of test data, our confidence for doing prediction on the test data with the trained model diminishes. Robust empirical optimization has been used to address this problem \cite{Bertsimas2018RobustSA,Bertsimas2018DatadrivenRO,BenTal2013RobustSO}. The work in \cite{Gotoh2018RobustEO} formulated a distributionally robust optimization (DRO) problem based on a minimax problem, where a model is trained against the worst distribution shifts between the train and test data:

\begin{align}\label{eq:DRO}
    \min_{\thetav} \max_{\mathbb{Q}} \bigg\{\mathbb{E}_{(x,y)\sim\mathbb{Q}} [\ell(h(x,\thetav), y)]-\frac{1}{\delta} \mathcal{H}_{\phi}(\mathbb{Q}|\hat{\mathbb{P}}_n)\bigg\},
\end{align}

where constant $\delta > 0$ is the regularization constant, and $\hat{\mathbb{P}}_n$ and ${\mathbb{Q}}$ are the train data empirical distribution and the test data distribution. The above problem optimizes against the ``worst-case" test distribution $\mathbb{Q}$, which does not diverge too much from $\hat{\mathbb{P}}_n$: the divergence of the distribution $\mathbb{Q}$ from $\hat{\mathbb{P}}_n$ is penalized in the regularization term $\frac{1}{\delta} \mathcal{H}_{\phi}(\mathbb{Q}|\hat{\mathbb{P}}_n)$, where $\mathcal{H}_{\phi}$ is a divergence measure between two distributions. The solution to this optimization problem is a model which is robust against distribution shifts between the train and test data, and its robustness increases with $\delta$. It was shown in \cite{Gotoh2018RobustEO} that the above DRO problem is equivalent to a mean-variance problem, where the empirical average loss is regularized with sample variance of the loss on the empirical train distribution $\hat{\mathbb{P}}_n$:

\begin{align}\label{eq:DRO_equivalence}
    &\min_{\thetav} \max_{\mathbb{Q}} \bigg\{\mathbb{E}_{(x,y)\sim\mathbb{Q}} [\ell(h(x,\thetav), y)]-\frac{1}{\delta} \mathcal{H}_{\phi}(\mathbb{Q}|\hat{\mathbb{P}}_n)\bigg\} \equiv\nonumber \\ 
    & \min_{\thetav} \bigg\{\mathbb{E}_{(x,y)\sim\hat{\mathbb{P}}_n} \bigg[\ell(h(x,\thetav), y)\bigg] + \frac{\delta}{2\phi''(1)}~ \mathbb{E}_{(x,y)\sim\hat{\mathbb{P}}_n} \bigg[\ell(h(x,\thetav), y) - \mathbb{E}_{(x,y)\sim\hat{\mathbb{P}}_n} [\ell(h(x,\thetav), y)]\bigg]^2\bigg\}.
\end{align}

This means that variance regularization can improve out-of-sample (test) performance. \cite{Maurer2009,Namkoong2017} proposed regularizing the empirical risk minimization (ERM) by the empirical variance of losses across training samples to balance bias and variance and improve out-of-sample (test) performance and convergence rate. Similarly, \cite{shivaswamy2010a} proposed boosting binary classifiers based on a variance penalty applied to exponential loss. 

DRO is also an effective approach to deal with imbalanced and non-iid data. Unlike the above sample-wise variance regularization works, the work \cite{Krueger2021OutofDistributionGV} - assuming having access to data from multiple training domains - proposed penalizing variance of training risks across the domains as a method of distributionally robust optimization to provide out-of-distribution (domain) generalization. The first work propopsing DRO in \FL setting is \cite{MohriSS19}, where they minimize the maximum combination of clients' local losses to address fairness in FL: $\min_{\thetav} \max_{i} f_i(\thetav)$. Also, the work in \cite{Deng2020DistributionallyRF} proposed a communication-efficient algorithm which performs well over the worst-case combination of clients' empirical local distributions:

\begin{align}\label{eq:DRFA_DRO}
    \min_{\thetav} \max_{\boldsymbol{\lambda}\in \Lambda} \sum_{i=1}^n \lambda_i f_i(\thetav),
\end{align}

where $\boldsymbol{\lambda}\in\Lambda=\{\lambda \in \mathbb{R}_+^n: \sum_{i=1}^n \lambda_i=1\}$. 
The relation between robust optimization and variance regularization in non-\FL settings (\cref{eq:DRO_equivalence}) encourages us to interpret \VRed as an equivalent form of DRO. Hence, the variance regularization connects \VRed non-trivially to the previous works \AFL \cite{MohriSS19} and \texttt{DRFA} \cite{Deng2020DistributionallyRF} through DRO.

\end{document}